\documentclass[letterpaper, journal, twoside]{IEEEtran} 
\usepackage[utf8]{inputenc}
\usepackage{cite}
\usepackage[english]{babel}
\usepackage{csquotes}
\usepackage{hyperref}
\usepackage{fancyhdr} 
\usepackage{textcomp}

\usepackage{enumitem}   
\usepackage{multicol}
\usepackage{multirow}

\usepackage{amsthm} 
\usepackage{amsmath}
\usepackage{mathtools}
\usepackage{amssymb}
\usepackage[percent]{overpic}
\usepackage{siunitx}
\usepackage{subfigure}
\usepackage{booktabs}
\usepackage{datetime}
\newdateformat{monthyeardate}{%
  \monthname[\THEMONTH], \THEDAY, \THEYEAR}
\usepackage[normalem]{ulem}

\usepackage{amsmath}
\usepackage{amssymb}

\newcommand{\cosa}[1]{\mathrm{c}_{#1}}
\newcommand{\sina}[1]{\mathrm{s}_{#1}}

\newcommand{\vect}[1]{\boldsymbol{#1}}
\newcommand{\mat}[1]{\boldsymbol{#1}}

\newcommand{\diff}[2]{ \displaystyle \frac{\partial #1}{\partial #2}}

\newcommand{\diffT}[2]{ \left( \displaystyle \frac{\partial #1}{\partial #2} \right)^{T} }

\newcommand{\diffs}[3]{\frac{\partial^2 #1}{
		\ifx#2#3 
		\partial #2^2
		\else
		\partial #2 \partial #3
		\fi
}}

\newcommand{\norm}[1]{{\left \| {#1} \right \|}}

\newcommand{\R}{\mathbb{R}}

\newenvironment{carray}
{ \left( \begin{array}}
	{ \end{array} \right) }

\newcommand{\zerov}{\vect{0}}

\newcommand{\dv}{\vect{d}}
\newcommand{\tildv}{\tilde{\vect{d}}}

\newcommand{\fv}{\vect{f}}
\newcommand{\gv}{\vect{g}}

\newcommand{\hv}{\vect{h}}
\newcommand{\kv}{\vect{k}}

\newcommand{\pv}{\vect{p}}

\newcommand{\qv}{{\vect{q}}}
\newcommand{\dqv}{\dot{\vect{q}}}

\newcommand{\rv}{{\vect{r}}}

\newcommand{\uv}{\vect{u}}
\newcommand{\vv}{\vect{v}}

\newcommand{\xv}{\vect{x}}

\newcommand{\yv}{\vect{y}}

\newcommand{\dyv}{\dot{\vect{y}}}

\newcommand{\tv}{\vect{t}}

\newcommand{\alphav}{\vect{\alpha}}
\newcommand{\betav}{\vect{\beta}}
\newcommand{\gammav}{\vect{\gamma}}

\newcommand{\phiv}{\vect{\phi}}

\newcommand{\tauv}{\vect{\tau}}

\newcommand{\thetav}{\vect{\theta}}
\newcommand{\dthetav}{\dot{\vect{\theta}}}

\newcommand{\xiv}{\vect{\xi}}

\newcommand{\hxiv}{\hat{\vect{\xi}}}

\newcommand{\Phim}{\vect{\Phi}}

\newcommand{\IIm}{\mat{I}}

\newcommand{\zerom}{\mat{O}}

\newcommand{\Am}{\mat{A}}

\newcommand{\Jm}{\mat{J}}

\newcommand{\Km}{\mat{K}}

\newcommand{\Pm}{\mat{P}}

\newcommand{\Sm}{\mat{S}}

\newcommand{\Sigmam}{\mat{\Sigma}}

\newcommand{\pvone}[1]{ \vect{\nabla}_{#1} }
\newcommand{\pvtwo}[2]{ \vect{\nabla}_{#2}\left(#1\right) }
\newcommand{\parv}[2]{\expandafter\ifx\expandafter\relax
	\detokenize{#1}\relax\pvone{#2}\else\pvtwo{#1}{#2}\fi}
	
\newcommand{\Prm}{\mathrm{P}}
\newcommand{\Drm}{\mathrm{D}}

\newcommand{\drm}{\mathrm{d}}

\newcommand{\rb}[1]{\left( #1 \right)}

\newenvironment{sequation*}
    {\begin{equation*}\small
    }
    { 
    \end{equation*}
    }

\newtheorem{theorem}{Theorem}
\newtheorem{property}{Property}
\newtheorem{corollary}{Corollary}

\theoremstyle{definition}
\newtheorem*{remark}{Remark}
\newtheoremstyle{named}{}{}{\itshape}{}{\bfseries}{.}{.5em}{\thmnote{#3}}
\theoremstyle{named}
\newtheorem*{namedassumption}{}
\newcounter{example}
\newenvironment{example}[1][]{\refstepcounter{example}\par\medskip
   \noindent \textbf{Example~\theexample#1.} \rmfamily}{\hfill$\vartriangleleft$\newline\par}
\NewDocumentCommand{\ActM}{ O{{}} O{{}} O{(\qv)}}{\Am^{#2\!}_{#1}#3}
\newcommand{\percentwidth}{1}

\begin{document}

\title{\LARGE \bf
Input Decoupling of Lagrangian Systems \\ via Coordinate Transformation: \\
{\large General Characterization and its Application to Soft Robotics}\\
}

\author{Pietro Pustina$^{1, 2}$, Cosimo Della Santina$^{2, 3}$, Frédéric Boyer$^{4}$, Alessandro De Luca$^{1}$, Federico Renda$^{5}$
\thanks{Manuscript received June, 12, 2023; revised October, 16, 2023; accepted January, 23, 2024. Date of publication \monthyeardate\today; date of current version \monthyeardate\today. This paper was recommended for publication by Editor P. Robuffo Giordano upon evaluation of the Reviewers' comments.}%
\thanks{The research of Pietro Pustina and Federico Renda was financially supported by the US Office of Naval Research Global under Grant N62909-21-1-2033, and in part by the Khalifa University of Science and Technology under Grants CIRA-2020-074, RC1-2018-KUCARS. The research of Cosimo Della Santina was financially supported by the Horizon Europe Program from Project EMERGE under Grant 101070918. The research of Alessandro De Luca was financially supported by the PNRR MUR project under Grant PE0000013-FAIR. \textit{(Corresponding author: Pietro Pustina.)}}
\thanks{$^{1}$Department of Computer, Control and Management Engineering, Sapienza University of Rome, 00185 Rome, Italy.
    {\footnotesize pustina@diag.uniroma1.it, deluca@diag.uniroma1.it}}
\thanks{$^{2} $Department of Cognitive Robotics, Delft University of Technology, Delft, The Netherlands. 
        {\footnotesize c.dellasantina@tudelft.nl}}%
\thanks{$^{3} $Institute of Robotics and Mechatronics, German Aerospace Center (DLR), 82234 Oberpfaffenhofen, Germany.}%
\thanks{$^{4} $ Institute Mines Telecom Atlantique, 44307 Nantes, France.
    {\footnotesize frederic.boyer@imt-atlantique.fr}}%
\thanks{$^{5}$Khalifa University Center for Autonomous
Robotics System (KUCARS) and the Department of Mechanical and Nuclear Engineering, Khalifa University of Science and Technology, 127788 Abu Dhabi, UAE.
    {\footnotesize federico.renda@ku.ac.ae}}
\thanks{Digital Object Identifier (DOI): 10.1109/TRO.2024.3370089}
}

\markboth{IEEE Transactions on Robotics}%
{Pustina \MakeLowercase{\textit{et al.}}: Input Decoupling of Lagrangian Systems via Coordinate Transformation}

\fancypagestyle{IEEEMark}{ 
\fancyhf{} 
\fancyfoot[L]{\footnotesize \copyright2024 IEEE. Personal use of this material is permitted.  Permission from IEEE must be obtained for all other uses, in any current or future media, including reprinting/republishing this material for advertising or promotional purposes, creating new collective works, for resale or redistribution to servers or lists, or reuse of any copyrighted component of this work in other works.} \renewcommand{\footrulewidth}{0pt} \renewcommand{\headrulewidth}{0pt} 
}

\maketitle
\thispagestyle{IEEEMark}

\begin{abstract}
Suitable representations of dynamical systems can simplify their analysis and control. On this line of thought, this paper aims to answer the following question: \textit{Can a transformation of the generalized coordinates under which the actuators directly perform work on a subset of the configuration variables be found?} Not only we show that the answer to this question is \textit{yes}, but we also provide necessary and sufficient conditions. More specifically, we look for a representation of the configuration space such that the right-hand side of the dynamics in Euler-Lagrange form becomes $[\IIm \; \zerom]^{T}\uv$, being $\uv$ the system input. We identify a class of systems, called \textit{collocated}, for which this problem is solvable. Under mild conditions on the input matrix, a simple test is presented to verify whether a system is collocated or not. By exploiting power invariance, we provide necessary and sufficient conditions that a change of coordinates decouples the input channels if and only if the dynamics is collocated. In addition, we use the collocated form to derive novel controllers for damped underactuated mechanical systems. To demonstrate the theoretical findings, we consider several Lagrangian systems with a focus on continuum soft robots.
\end{abstract}
%
\begin{IEEEkeywords}
Underactuated Robots; Dynamics; Motion Control; Modeling, Control, and Learning for Soft Robots.
\end{IEEEkeywords}
\section{Introduction}
\IEEEPARstart{E}{lectrical}, hydraulic, and mechanical systems, or their combinations, are Lagrangian systems that usually exhibit complex behavior. However, their physical nature displays special properties, such as symmetry and passivity, which have been exploited to solve many control problems~\cite{chen2010distributed, ortega2013passivity, loria2015observers, ge2019hierarchical, sun2020distributed}, otherwise difficult to address for generic nonlinear dynamics. To cope with their high nonlinearity and large number of degrees of freedom (DOF), representations with specific structures play a crucial role in simplifying analysis, as well as control design and synthesis. For example, coordinate transformations are often used to highlight some internal structure that simplify derivation of feedback controllers for robotic systems and prove their stability~\cite{giordano2019coordinated, yi2020path, mengacci2021motion, keppler2022underactuation}. 
\begin{figure}[t]
    \centering
    \includegraphics[width = \percentwidth\columnwidth]{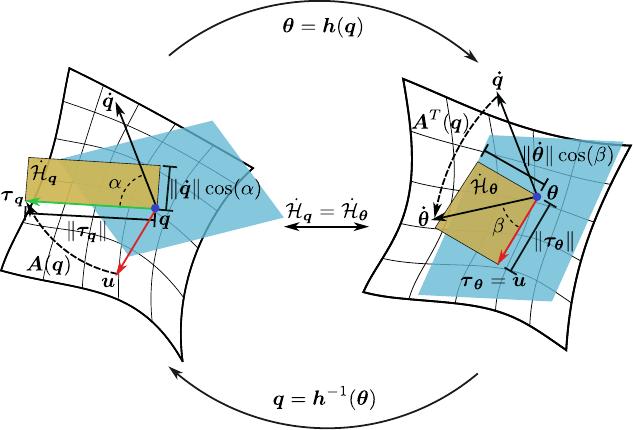}
    \caption{Graphical representation of the proposed change of coordinates addressing the input decoupling problem for a fully actuated Lagrangian system. In the $\qv$ coordinates, the input $\uv$ undergoes a nonlinear transformation through the actuation matrix $\ActM$ when performing work on $\qv$, i.e., $\tauv_{\qv} = \ActM \uv$. The proposed change of coordinates $\thetav = \hv(\qv)$ bends the configuration space so that each component of $\uv$ acts directly on one component of $\thetav$, namely $\tauv_{\thetav} = \uv$. The existence of this transformation is possible because of the conservation of power $\dot{\mathcal{H}}_{\qv}$, represented by the yellow area, under change of coordinates, i.e., $\dot{\mathcal{H}}_{\qv}=\dot{\mathcal{H}}_{\thetav}$.
    }
    \label{fig:coordinate_change}
\end{figure}
This paper considers the Input Decoupling (ID) problem for input-affine Lagrangian systems. In particular, given a Lagrangian system whose inputs enter the second-order equations of motion through a configuration-dependent actuation matrix, we study under which conditions on the definition of generalized coordinates each input affects one and only one equation of motion.
The coordinates solving the ID problem constitute a set of variables that simplify control design and synthesis because the input directly affects the equations of motion in a decentralized form. If the decoupling coordinates are regarded as system outputs, any control law implemented in these coordinates automatically decouples also the input-output channels.
For fully actuated or overactuated Lagrangian systems, i.e., when the number of independent inputs is equal or larger than the number of generalized coordinates, the ID problem can be solved by a configuration-dependent input transformation that inverts (or pseudoinverts) the actuation matrix. On the other hand, in the case of underactuated systems, a solution is available so far if the actuation matrix is constant and requires only a linear change of generalized coordinates~\cite{borja2022energy}. 

For multi-input-multi-output (MIMO) plants described in state space form, the input-output decoupling problem is solved by an inversion-based feedback controller, both in the linear~\cite{zhang2016multivariable} and nonlinear case~\cite{krener1973equivalence, brockett1978feedback, isidori1981nonlinear}. For this, a feedback transformation is needed that requires a state and input transformation together with a feedback action.
In particular, input-output decoupling is possible by a static state feedback if and only if the system has well defined vector relative degree, namely a nonsingular decoupling matrix~\cite{nijmeijer1990nonlinear, isidori1992nonlinear}. This result can be used to solve the ID problem for Lagrangian systems by taking the configuration variables (or a subset of them) as system outputs. In fact, in this way, the system has a vector relative degree, with all outputs having uniform relative degree two. It can be shown that, in case of underactuation, the necessary and sufficient condition is the involutivity of the distribution spanned by the columns of the actuation matrix~\cite{ortega2021pid}.
Unfortunately, involutivity is not easy to check, especially for high dimensional dynamics, and has to be evaluated case by case, see~\cite{wu2019adaptive, lazrak2018improved, mishra2018design, ding2022robust}. Furthermore, finding the input decoupling state variables in this case requires solving a system of nonlinear partial differential equations~\cite[Chap.~5]{isidori1992nonlinear}, which could be impractical for control synthesis.
Finally, in~\cite{skogestad2023transformed}, the authors have proposed input transformations for dynamic processes that achieve exact linearization and input decoupling, under complete model knowledge. However, the analysis is restricted to systems with equal number of states, inputs and outputs.

In this paper, we focus on Lagrangian systems and we show that the choice of particular coordinates, called \textit{actuation coordinates}, solves the ID problem without requiring a configuration-dependent transformation of the input nor a state-feedback. 
We derive necessary and sufficient conditions under which the actuation coordinates exist and constructively show that these coordinates solve the ID problem for fully actuated, overactuated, and underactuated Lagrangian systems. In particular, the unactuated coordinates remain arbitrary when the system is underactuated.

These results stem from power invariance under change of coordinates, as graphically illustrated in Fig.~\ref{fig:coordinate_change} for the fully actuated case. Remarkably, similar considerations hold also when the dynamics is underactuated. 
We apply the results to several mechanical examples as archetypal Lagrangian systems. In addition, we prove that robotic systems driven by thread-like actuators, such as inelastic tendons or thin fluidic chambers, always admit actuation coordinates. 
Our results have relevant consequences on the control of continuum soft robots and other underactuated mechanical systems. Indeed, recent control laws for planar underactuated soft robots~\cite{franco2021energy, caasenbrood2021energy, pustina2022feedback, borja2022energy, pustina2023psatid} generalize to all collocated mechanical systems with damping, such as soft robots moving in 3D. 
Interestingly, the energy-based regulator in~\cite{soleti2023energy} uses in fact actuation coordinates to decouple the equations of motion of a 3-DOF underactuated soft robotic system. 
Similarly, in flexible link robots, the actuation coordinates are the clamped angles at the base of each beam, which have been extensively used in control design~\cite{deluca2021flexible}.

The paper is organized as follows. In Sec.~\ref{sec:preliminaries}, we introduce the notation and formalize the ID problem. Section~\ref{sec:collocated lagrangian systems} defines the class of Lagrangian systems, called \textit{collocated}, for which the ID problem is solvable. Necessary and sufficient conditions for solving the ID problem are then derived for fully actuated or overactuated dynamics (Sec.~\ref{sec:fully actuated systems}), and for underactuated systems (Sec.~\ref{sec:underactuated systems}). 
In Sec.~\ref{sec:thread-like actuators}, we prove that thread-like actuators yield collocated mechanical systems. Section~\ref{sec:example} extends two control strategies derived for underactuated mechanical systems with constant actuation matrix to the collocated case, validating one of these controllers on a 3D tendon-driven underactuated soft robot. Finally, conclusions and future works are summarized in Sec.~\ref{sec:conclusions}. 
%
%
\section{Preliminaries}\label{sec:preliminaries}
\subsection{Notation}
We denote vectors and matrices with bold letters. 
Arguments of the functions are omitted when clear from the context. 
Table~\ref{tab:notation} presents the notation adopted in the paper. 
\begin{table}[t!]\caption{Nomenclature}
\centering
\begin{tabular}{p{0.3\columnwidth} p{0.65\columnwidth} }
\toprule
\toprule
Symbol & Description\\
\toprule
$\R^{n}$ & Euclidean space of dimension $n$\\
$\mathbb{S}^{n}$ & Unit sphere of dimension $n$\\
$\R^{n \times m}$ & Space of $n \times m$ matrices over $\R$\\
$\mathcal{X}$ & Manifold of interest with $\mathcal{X} = \{ \mathcal{M}, \mathcal{N} \}$\\
$\R^{>0}$ & Positive real numbers $n$\\
$T_{\vv}\mathcal{X}$ & Tangent space of manifold $\mathcal{X}$ at $\vv \in \mathcal{X}$\\
$\mathcal{B}(\vv)$ & Neighbourhood of $\vv \in \mathcal{X}$ \\
$\mathfrak{se}(3)$ & Special Euclidean algebra of dimension $3$\\
$\mathfrak{so}(3)$ & Special orthogonal algebra of dimension $3$\\
$\IIm_{n} \in \R^{n \times n}$ & Identity matrix of dimension $n$\\
$\zerom_{n \times m}$ & Zero matrix of dimension $n \times m$\\
$\Pm > 0$ & Symmetric positive definite matrix\\
$\Sm_{i} \in \R^{n}$ & Column $i$ of matrix $\Sm \in \R^{n \times m}$\\
$S_{ij} \in \R$ & Element in row $i$ and column $j$ of $\Sm$\\
$[\vv]_{i} \in \R^{i}$ & Vector containing the first $i$ components of $\vv \in \R^{n}$, with $i \leq n$\\
$\norm{\vv}$ & Euclidean norm of $\vv$\\
$\tilde{\rv} \in \mathfrak{so}(3)$ & Skew symmetric matrix defined by $\rv \in \R^{3}$\\
$\hat \alphav \hspace{-2pt}  = \hspace{-2pt} \small \begin{carray}{@{\hspace{-2pt}}cc@{\hspace{-2pt}}}
    \tilde \betav & \gammav \\
    \zerov & 0
\end{carray} \hspace{-2pt} \in \hspace{-2pt} \mathfrak{se}(3)$ & Tensor representation of $\alphav \hspace{-2pt}= \hspace{-2pt}\small \rb{\betav^{T}\ \hspace{-2pt}\gammav^{T}}^{T}$ with $\betav, \gammav \in \R^3$\\
$\Jm_{\fv}(\xv)\hspace{-2pt} =\hspace{-2pt} \diff{\fv}{\xv} \hspace{-2pt} \in\hspace{-2pt} \mathbb{R}^{l \times h}$ & Jacobian of the vector function $\fv(\xv) : \mathbb{R}^{h} \hspace{-3pt} \rightarrow \hspace{-3pt} \mathbb{R}^{l}$\\
$\boldsymbol{\mathrm{tanh}}(\vv)$ & Vector obtained by applying $\mathrm{tanh}(\cdot)$ componentwise to $\vv$\\
\bottomrule
\bottomrule
\end{tabular}
\label{tab:notation}
\end{table}
\subsection{Dynamic model}
Let $\qv \in \mathcal{M}$ be the generalized coordinates of a dynamical system evolving on a $n$-dimensional smooth manifold $\mathcal{M}$ with Lagrangian $\mathcal{L}_{\qv}(\qv, \dqv)$. The system trajectories satisfy the Euler-Lagrange equations of motions
\begin{equation}\label{eq:lagrangian dynamics}
\left\{
\begin{array}{l}
         \displaystyle\frac{\mathrm{d}}{\mathrm{d}t}\diffT{\mathcal{L}_{\qv}(\qv, \dqv)}{\dqv} - \diffT{ \mathcal{L}_{\qv}(\qv, \dqv) }{\qv} = \tauv_{\qv}(\qv, \uv),\\
         \tauv_{\qv}(\qv, \uv) = \ActM\uv,
\end{array}
\right.
\end{equation}
where $\uv \in \R^{m}$ are the available actuation inputs, $\ActM \in \R^{n \times m}$ is the actuation matrix, and $\tauv_{\qv}(\qv, \uv) \in \mathrm{Im}(\ActM)$ collects the generalized forces performing work on $\qv$.
For all $\qv \in \mathcal{M}$, we assume that $\ActM$ is a full-rank matrix, i.e., $r = \mathrm{rank}(\ActM) = \min(m, n)$. 
When the dynamics is fully actuated ($m = n$) or underactuated ($m < n$), this is equivalent to asking that the actuation channels are all independent. On the other hand, if~\eqref{eq:lagrangian dynamics} is overactuated ($m > n$), we assume that there are exactly $n$ independent inputs. 
Note that, for the following derivations, when $r = m^{*} < \min(m, n)$ one can consistently discard $m-m^{*}$ linearly dependent columns of $\ActM$ and consider the dynamics as underactuated. 

In Appendix~\ref{appendix:properties}, we recall two basic properties of Lagrangian systems used in the following results. 
\subsection{Problem statement}\label{section:problem statement}
We look for a change of coordinates $\thetav = \hv(\qv) $ from $\mathcal{B}(\qv) \subset \mathcal{M}$ to $\mathcal{N}$ where each of the first $r$ equations of motion in~\eqref{eq:lagrangian dynamics} is affected by one, and only one, independent actuator input, i.e., the right-hand side of the transformed equations of motion (see Property~\ref{property:change of coordinates} in Appendix~\ref{appendix:properties}) takes the form
\begin{equation}\label{eq:tau theta}
    \tauv_{\thetav}(\thetav, \uv) = \ActM[\thetav][][(\thetav)]\uv = \begin{carray}{c}
        \IIm_{r}\\
        \zerom_{n-r \times r}
    \end{carray} \uv.
\end{equation}
We refer to such problem as the~\textit{Input Decoupling (ID) problem} for the Lagrangian dynamics~\eqref{eq:lagrangian dynamics}. If a solution exists, then we say that~\eqref{eq:lagrangian dynamics} admits a~\textit{collocated form}. Note that~\eqref{eq:tau theta} covers only the fully- and underactuated cases. When the dynamics is overactuated, i.e., $m > n$, it is impossible to obtain~\eqref{eq:tau theta} because only $r = n$ input channels can be decoupled, and the remaining $m-n$ inputs will affect the dynamics through a configuration-dependent actuation matrix. 
%
\section{Collocated Lagrangian Systems}\label{sec:collocated lagrangian systems}
In this section, we characterize a new class of Lagrangian systems, which we call \textit{collocated} because only such Lagrangian dynamics admit a collocated form under a change of generalized coordinates. In addition, a set of coordinates that solve the ID problem come for free without further system analysis. 

To this end, we will exploit a concept that is known as the \emph{passive output} in the context of passivity-based control~\cite{tsolakis2021distributed}. 
Consider the following vector function linear in the velocity
\begin{equation*}\label{eq:passive output}
    \dyv = \Am^{T}(\qv) \dqv,
\end{equation*}
which is called the passive output because~\eqref{eq:lagrangian dynamics} is passive with respect to the pair $(\uv, \dyv)$, with the storage function being the system Hamiltonian (see Appendix~\ref{appendix:properties}).
We will assume that~\eqref{eq:lagrangian dynamics} has $\dyv$ integrable, i.e., 
\begin{namedassumption}[Integrability assumption]\label{assumption:actuation coordinates}
For all $\qv \in \mathcal{M}$, there exists a function $\gv(\qv) : \mathcal{M} \rightarrow \R^{m}$ such that 
\begin{equation}\label{eq:integrability assumption}
    \Jm_{\gv}(\qv) = \diff{\gv}{\qv} = \Am^{T}(\qv).
\end{equation}
\end{namedassumption}
If the passive output is integrable, then we say that the Lagrangian system~\eqref{eq:lagrangian dynamics} is collocated because it admits a collocated form as defined in Section~\ref{section:problem statement}.
Furthermore, we define $\yv = \gv(\qv)$ as \emph{actuation coordinates} because, in such coordinates, $\uv$ acts directly on the equations of motion according to~\eqref{eq:tau theta}.
\begin{remark}
Each component of $\gv(\qv)$ is defined up to a constant since any function $\bar{\gv}(\qv) = \gv(\qv) + \kv$, with $\kv \in \R^{m}$, satisfies the condition $\Jm_{\Bar{\gv}} = \Jm_{\gv}$.  
\end{remark}
The integrability assumption requires each column of $\ActM$ to be the gradient of a scalar function of the configuration variables. If $\ActM$ is constant, i.e., $\ActM = \ActM[][][]$, then $\gv(\qv) = \Am^{T}\qv$.
More in general, when the column $\Am_{i}(\qv)$ has continuous partial derivatives, $\dot{y}_{i} = \ActM[i][T]\dqv$ is integrable~\cite[Chap.~2]{do1998differential} if and only if 
\begin{equation}\label{eq:integrability condition}
    \diff{A_{ji}}{q_{k}} = \diff{A_{ki}}{q_{j}}; \quad \forall j,k \in \{1, \cdots, n\}.
\end{equation}
Note that this condition is equivalent to asking that, when $u_i$ is constant, the generalized work done by $u_{i}$ on $\qv$ does not depend on the system trajectories but only on the initial and final configurations $\qv_a$ and $\qv_b$, respectively, i.e.,  
\begin{equation*}
    W_{u_{i}}(\qv) := \int^{\qv_{b}}_{\qv_{a}} u_{i} \ActM[i][T] d\qv = u_{i}\left[g_{i}(\qv_{b}) - g_{i}(\qv_{a})\right], 
\end{equation*}
where the last equality follows from the Gradient Theorem~\cite[Prop.~1, Chap.~2]{do1998differential}. 
In other words, $P_{i}(u_{i}, \qv) := u_{i}g_{i}(\qv)$ plays the role of a potential energy for the dynamics. If the actuation matrix is obtained using a differential formalism, such as the virtual works principle, it is reasonable to expect--although without any guarantee--that the integrability holds due to the inherent differentiation involved. It is also worth observing the following.
\begin{remark}
    If $\ActM[][T]$ is integrable according to~\eqref{eq:integrability assumption}, then the orthogonal complement to the co-distribution spanned by $\ActM[][T]$ satisfies the Frobenius theorem.
    However, in general, the inverse implication does not hold without also an input transformation. 
\end{remark}

Even when~\eqref{eq:integrability condition} is satisfied, it could be challenging to integrate the passive output $\dyv$ in closed form. Nonetheless, it is always possible to perform the numerical integration online based on the measure of $\qv$ and $\dqv$. Assuming an exact knowledge of $\ActM$ and neglecting integration errors, there is formally no difference in having $\yv$ in closed form or computing it online. Furthermore, in many cases, the numerical integration should not be necessary. This is because the actuation coordinates are inherently related to the system inputs and should be easily measurable.

In some cases, the integrability may come directly as a consequence of the physical nature of the system, as illustrated in the following.  
\begin{example}[~(Cartesian forces on a robot)]
Inspired by~\cite{deluca1998steering}, consider a manipulator with $n$-DOF subject to $m$ external forces applied on its structure. Assume that each force $\fv_{i}$ changes its magnitude over time but keeps the same direction in the global Cartesian frame so that
\begin{equation*}
    \fv_{i} = \dv_{\fv, i} u_{i}; \quad i \in \{1, \cdots, m\},
\end{equation*}
with $\norm{\dv_{\fv, i}} = 1$, and where $\dv_{\fv, i} \in \R^{3}$ represents the direction of $\fv_i$ and $u_{i} \in \R^{>0}$ its magnitude. If $\pv_{i}(\qv) \in \R^{3}$ denotes the point of application of $\fv_{i}$ in the global frame, then the effect of the force in the dynamics is 
\begin{equation*}
    \tau_{\qv, i} = \Jm^{T}_{\pv_i}(\qv)\fv_{i} = \Jm^{T}_{\pv_i}(\qv) \dv_{\fv, i} u_{i} = \Am_{i}(\qv)u_{i}. 
\end{equation*}
Thus, one can integrate $\Am_{i}(\qv)$ as $g_{i}(\qv) = \dv_{\fv, i}^{T}\pv_{i}(\qv)$.
\end{example}
On the other hand, the integrability conditions may not hold even for elementary dynamics. 
\begin{example}[~(Geostationary satellite)]
\begin{figure}
    \centering
    \includegraphics[width = 0.7\percentwidth\columnwidth]{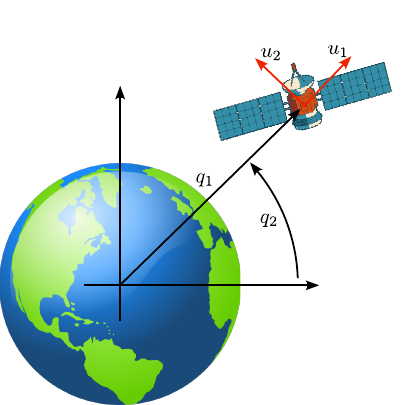}
    \caption{A geostationary satellite actuated by a normal force $u_1$ and a tangential force $u_2$. The body configuration is described by the distance $q_1$ from the Earth center and the angle $q_2$ with respect to the horizontal axis. Only the normal force is collocated because it performs work directly and only on $q_1$.}
    \label{fig:example:satellite}
\end{figure}
Consider a geostationary satellite orbiting around the Earth in a plane. The configuration $\qv = (q_1\,\,q_2)^{T}$ describes its motion, being $q_1$ the distance from the Earth center and $q_2$ the angle with respect to a reference axis, as shown in Fig.~\ref{fig:example:satellite}. The satellite is controlled by a normal and tangential force, denoted as $u_{1}$ and $u_{2}$, respectively. 
The application of the Lagrangian formalism yields
\begin{equation*}
    \tauv_{\qv} = \ActM \uv = \begin{carray}{cc}
        1 & 0\\
        0 & q_1
    \end{carray}\uv.
\end{equation*}
The first column of $\ActM$ is already in the collocated form because $u_{1}$ performs work only on $q_1$. However, the second column of $\ActM$ is non-integrable because
\begin{equation*}
    \diff{A_{12}}{q_{2}} = 0 \neq 1 = \diff{A_{22}}{q_{1}}.
\end{equation*}
Indeed, the generalized work done on $\qv$ by the tangential force $u_2$ depends on the time evolution of $q_{1}$.
\end{example}
In the following, we consider the three different actuation scenarios, namely fully actuated, overactuated and underactuated dynamics.  
\section{Change of Coordinates\\for Fully Actuated Systems}~\label{sec:fully actuated systems}
In this section, we show that when the system is fully actuated, the existence of $n$ actuator coordinates is necessary and sufficient to bring the dynamics to the collocated form. This case allows illustrating the results in the most simple scenario, where an input transformation is sufficient to solve the ID problem. However, the same arguments will be used also for underactuated systems.
\begin{theorem}\label{theorem:fully actuated}
Suppose the system is fully actuated, i.e., $m = n$. There exists a change of coordinates $\thetav = \hv(\qv) : \mathcal{B}(\qv) \rightarrow \mathcal{N}$ such that~\eqref{eq:lagrangian dynamics} takes the form
\begin{equation}\label{eq:coordinate change:fully actuated eom}
    \frac{\mathrm{d}}{\mathrm{d}t}\diffT{\mathcal{L}_{\thetav}(\thetav, \dthetav)}{\dthetav} - \diffT{ \mathcal{L}_{\thetav}(\thetav, \dthetav) }{\thetav} = \uv,
\end{equation}
i.e., $\ActM[\thetav][][(\thetav)] = \IIm_{n}$, if and only if the integrability assumption~\eqref{eq:integrability assumption} holds.
Let $\gv(\qv)$ be the integral of $\ActM[][T]\dqv$. Then, a possible choice for $\thetav$ is $\thetav = \gv(\qv)$. 
\end{theorem}
\begin{proof}
Under the integrability assumption, $\thetav = \gv(\qv)$ defines a change of coordinates because its Jacobian $\Jm_{\gv}(\qv) = \Am^{T\!}(\qv)$ has rank $n$ at $\qv$.

Since the generalized power is coordinate invariant (Property~\ref{property:change of coordinates} in Appendix~\ref{appendix:properties}), it follows
\begin{equation*}
    \dthetav^{T}\tauv_{\thetav} = \dqv^{T}\tauv_{\qv}.
\end{equation*}
Noting that $\dthetav = \Jm_{\gv}(\qv)\dqv$ and using $\tauv_{\qv} = \ActM \uv$, the above equation rewrites as
\begin{equation*}
    \dqv^{T}\Jm^{T}_{\gv}(\qv) \tauv_{\thetav} = \dqv^{T} \ActM\uv,
\end{equation*}
or, equivalently,
\begin{equation}\label{eq:proof:fully actuated: power balance}
    \dqv^{T}\rb{ \Jm^{T}_{\gv}(\qv) \tauv_{\thetav} - \ActM\uv } = 0.
\end{equation}
Since~\eqref{eq:proof:fully actuated: power balance} holds for all $\dqv \in T_{\qv}\mathcal{M}$, it follows that
\begin{equation*}
    \Jm^{T}_{\gv}(\qv) \tauv_{\thetav} - \ActM\uv = \zerov.
\end{equation*}
Furthermore, $\Jm_{\gv}(\qv) = \Am^{T}(\qv)$ leads to
\begin{equation*}
    \ActM \rb{\tauv_{\thetav} - \uv} = \zerov.
\end{equation*}
The above equation defines a homogeneous linear system in the unknown $\tauv_{\thetav} - \uv$, which admits the unique solution $\tauv_{\thetav} = \uv$ since $\ActM$ is nonsingular, thus yielding the sufficiency of~\eqref{eq:coordinate change:fully actuated eom}.

As for the necessity, suppose that a change of coordinates $\thetav = \hv(\qv)$ exists such that~\eqref{eq:coordinate change:fully actuated eom} holds. Property~\ref{property:change of coordinates} implies that, for all $\dqv \in T_{\qv}\mathcal{M}$,
\begin{equation*}
    \dqv^{T} \rb{\Jm^{T}_{\hv}(\qv) - \ActM } \uv = 0,
\end{equation*}
leading to
\begin{equation*}
    \rb{\Jm^{T}_{\hv}(\qv) - \ActM } \uv = \zerov ; \quad \forall \uv\in \R^{n}.
\end{equation*}
If one chooses $\uv = (\IIm_{n})_{i};\ i \in \{ 1, \cdots, n \}$, then 
\begin{equation*}
    \rb{\Jm^{T}_{\hv}(\qv) - \ActM } \uv = (\Jm^{T}_{\hv}(\qv) - \ActM)_{i} = \zerov.
\end{equation*}
Thus, it holds $\Jm^{T}_{\hv}(\qv) = \ActM$ and $\gv(\qv) = \hv(\qv)$.
\end{proof}
The following example illustrates the above result.
\begin{example}[~(Spring actuated mechanism)]
    Consider a planar mechanism with two passive revolute joints, having angles $q_1$ and $q_2$ so that $\qv = (q_1\,\, q_2)^{T}$. A spring with stiffness $k_{i}$ is attached to the distal end of each link, whose length is $l_{i},\ i = 1, 2$.
    The springs are also connected to two carts moving on linear rails under the forces $u_1$ and $u_2$, with reference to Fig.~\ref{fig:planar 2R spring}.  
    %
    \begin{figure}
        \centering
        \includegraphics[width = 0.9\percentwidth\columnwidth]{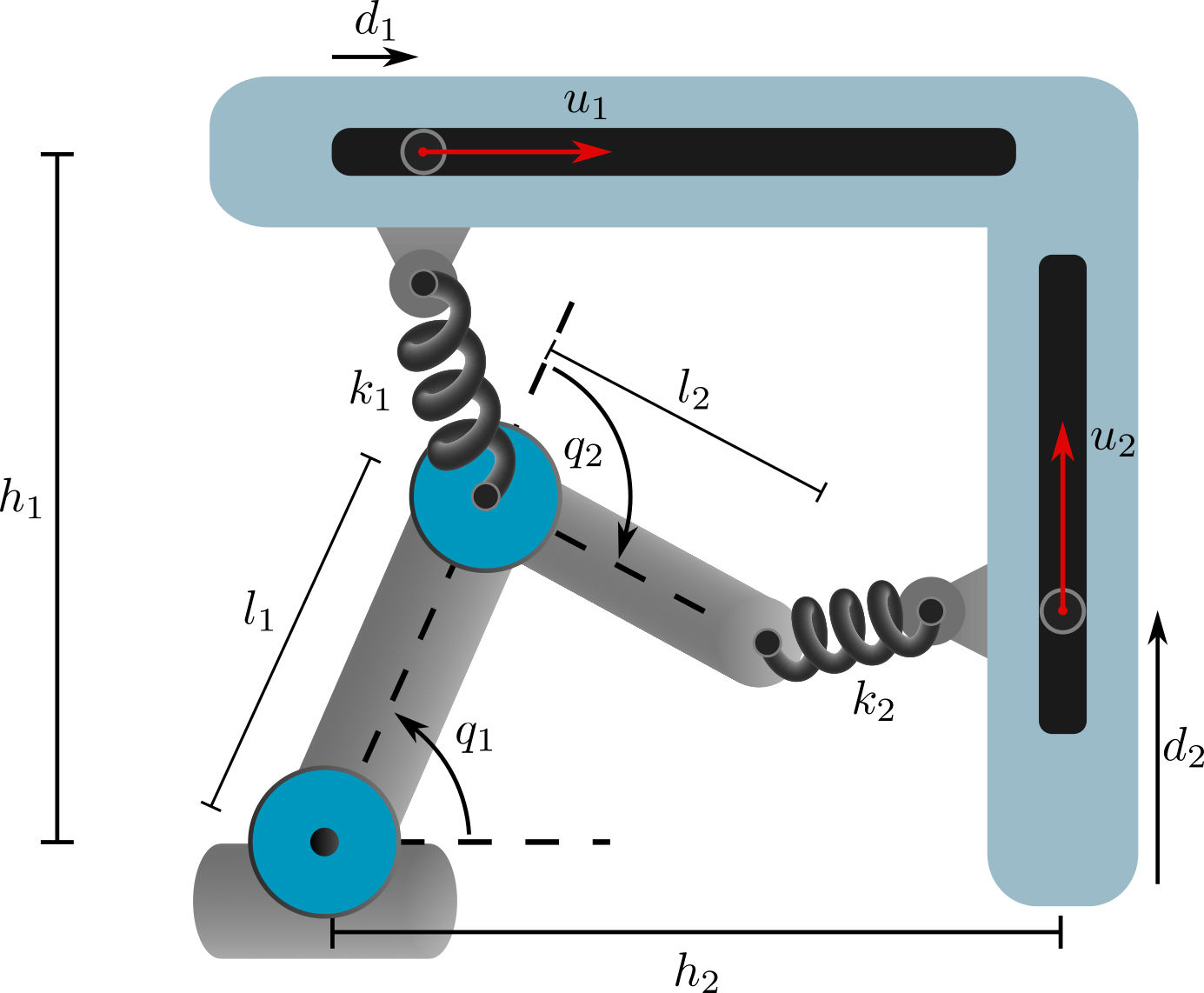}
        \caption{A planar mechanism with two rotational joints having configuration $\qv \in \mathbb{S}^{1} \times \mathbb{S}^{1}$. Positive (relative) rotations are counted counterclockwise. For $i = 1, 2$, the force $u_{i}$ actuates a cart coupled to the robot by a linear spring of stiffness $k_i$. The dynamics of the carts is negligible, and thus the forces $\uv$ act instantaneously on the mechanism.}
        \label{fig:planar 2R spring}
    \end{figure}
    Assuming that the dynamics of the carts is negligible, it can be shown that the actuator inputs $\uv$ directly affect the joint motion through the actuation matrix 
    \begin{equation*}
        \ActM = \begin{carray}{cc}
            -l_{1}\sina{1} & l_{1}\cosa{1} + l_{2}\cosa{12}\\
            0 & l_{2}\cosa{12}
        \end{carray},
    \end{equation*}
    where $\sina{i}(\cosa{i}) = \sin(q_i)(\cos(q_i))$ and $\sina{ij}(\cosa{ij}) = \sin(q_i + q_j)(\cos(q_i + q_j))$. 
    Outside of singularities $q_1 \in \{0, \pi\}$ and $q_1 + q_2 \in \{ \pm \pi/2\}$, $\mathrm{rank}(\ActM) = 2$ and the integrability test~\eqref{eq:integrability condition} is successful because
    \begin{align*}\small
        \diff{A_{11}}{q_2} &= \diff{\rb{-l_{1}\sina{1}}}{q_{2}} = 0 = \diff{A_{21}}{q_{1}},\\
        \diff{A_{21}}{q_2} &= \diff{\rb{l_{1}\cosa{1} + l_{2}\cosa{12}}}{q_{2}} = -l_{2}\sina{12}\\&= \diff{\rb{l_{2}\cosa{12}}}{q_{1}} = \diff{A_{22}}{q_1}.
    \end{align*}
    The passive output is integrable as
    \begin{equation*}
        \yv = \begin{carray}{c}
            l_{1}\cosa{1}\\
            l_{1}\sina{1} + l_{2}\sina{12}
        \end{carray}.
    \end{equation*}
    Note that $y_1$ is the $x$-coordinate of the position of the spring end attached to link $1$. Similarly, $y_2$ is the $y$-coordinate of the spring attached to link $2$.
    Indeed, the forces $u_{1}$ and $u_{2}$ perform work on the distal ends of the spring attached to the mechanism along these directions.  
\end{example}
\subsection{Overactuated case}
The previous result extends to overactuated systems, namely dynamics with more inputs than generalized coordinates. Thus, we have $r = n < m$. 
We partition $\ActM$, which is a wide matrix, as
\begin{equation}\label{eq:actuation matrix expansion}
    \ActM = \begin{carray}{cc}
        \ActM[a] & \ActM[o]
    \end{carray},
\end{equation}
where $\ActM[a] \in \R^{n \times n}$ and $\ActM[o] \in \R^{n \times (m-n)}$. 
Without loss of generality, we can have that $\mathrm{rank}\rb{\ActM[a]} = n$ and the integrability condition holds for $\ActM[a]$.
\begin{corollary}\label{corollary:overactuated}
If the system is overactuated, i.e., $r = n < m$ and the same hypotheses of Theorem~\ref{theorem:fully actuated} hold for $\ActM[a]$, then there exists a change of coordinates $\thetav = \hv(\qv):\mathcal{B}(\qv) \rightarrow \mathcal{N} $ such that~\eqref{eq:lagrangian dynamics} takes the form
\begin{equation}\label{eq:coordinate change:overactuated eom}
    \frac{\mathrm{d}}{\mathrm{d}t}\diffT{\mathcal{L}_{\thetav}(\thetav, \dthetav)}{\dthetav} - \diffT{ \mathcal{L}_{\thetav}(\thetav, \dthetav) }{\thetav} = \rb{ \IIm_{n} \quad \ActM[o, \thetav][][(\thetav)]} \uv,
\end{equation}
where
$$
\ActM[o, \thetav][][(\thetav)] = \ActM[a][-1][]\ActM[o][][\rb{\qv = \hv^{-1}(\thetav)}] \in \R^{n \times (m-n)}.
$$
If $\gv(\qv)$ is the integral of $\ActM[a][T]\dqv$, then $\thetav$ can be chosen as $\thetav = \gv(\qv)$. 
\end{corollary}
\begin{proof}
    Choosing again $\thetav = \gv(\qv)$ and following steps similar to those of the proof of Theorem~\ref{theorem:fully actuated}, one obtains
    \begin{equation*}
        \ActM[a] \tauv_{\thetav} = \ActM \uv.
    \end{equation*}
    Expanding $\ActM$ into~\eqref{eq:actuation matrix expansion} and left-multiplying the above equation by $\ActM[a][-1]$ gives
    \begin{equation*}
        \begin{split}
            \tauv_{\thetav} &= \ActM[a][-1] \begin{carray}{cc}
                \ActM[a] & \ActM[o]
            \end{carray}\uv\\
            &= (\,\,
                \IIm_{n} \quad \underbrace{\ActM[a][-1]\ActM[o]}_{\ActM[o, \thetav][][(\thetav)]}
               \,\,)
                \uv,
        \end{split}
    \end{equation*}
    with $\qv = \gv^{-1}(\thetav)$.

    Now, assume a change of coordinates $\thetav = \hv(\qv)$ exists such that~\eqref{eq:coordinate change:overactuated eom} holds. 
    After some computations, power invariance leads to the algebraic system
    \begin{equation*}
    \begin{carray}{cc}
        \Jm^{T}_{\hv}(\qv)(\qv)-\ActM[a] & \Jm^{T}_{\hv}(\qv)\ActM[o, \thetav] - \ActM[o]
    \end{carray}\uv = \zerov,
    \end{equation*}
    which must hold for all $\uv \in \R^{m}$. By taking $\uv_{i} = (\IIm_{m})_{i};\ i \in \{ 1, \cdots, m \}$, it follows $\Jm^{T}_{\hv}(\qv) = \ActM[a]$ and $\ActM[o] = \Jm^{T}_{\hv}(\qv) \ActM[o, \thetav]$. Hence, at least $n$ passive outputs are integrable as $\yv = \hv(\qv)$. 
\end{proof}
Note that it is not possible, in general, to simplify the expression of both terms in the actuation matrix because there are too many input variables to be decoupled. 
\begin{example}[~(Tendon driven joint)]\label{example:finger}
    Consider the tendon driven finger of~\cite[Chap.~6.4]{murray1994mathematical} with $1$-DOF $q$ and two actuator inputs $\uv = (u_1 \,\, u_2)^{T}$, as sketched in Fig.~\ref{fig:example:finger}. 
    \begin{figure}
        \centering
        \includegraphics[width=\percentwidth\columnwidth]{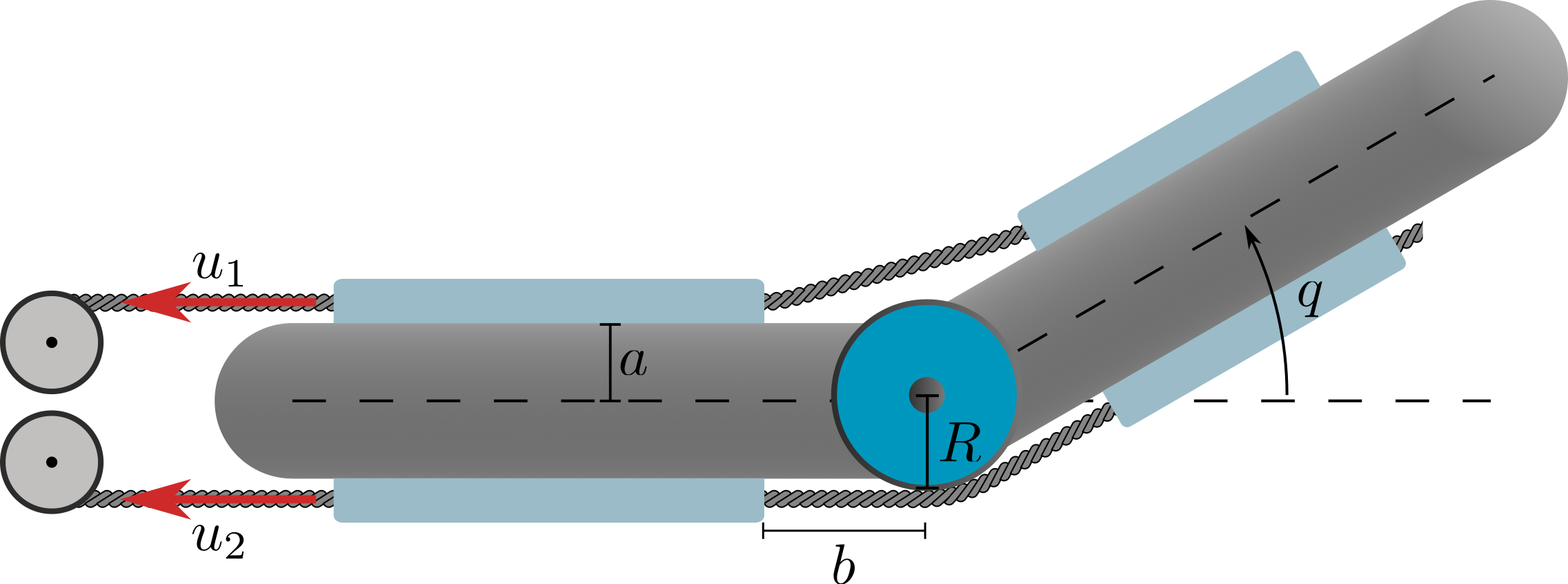}
        \caption{A tendon-driven rotational joint having a single configuration variable $q \in \mathbb{S}^{1}$. The cables tension $u_1$ and $u_2$ generate a torque at the joint.}
        \label{fig:example:finger}
    \end{figure}
    Assume for simplicity that the angle $q > 0$ (similar results hold for $q < 0$).
    The system input is $\tauv_{q} = \ActM[][][(q)]\uv$, where 
    \begin{equation*}
        \ActM[][][(q)] = \begin{carray}{@{\hspace{-1pt}}cc@{\hspace{-1pt}}}
            \displaystyle \sqrt{a^2 + b^2}\sin\rb{ \tan^{-1}\rb{\frac{a}{b}} + \frac{q}{2}} & -R
        \end{carray},
    \end{equation*}
    and $\uv$ collects the cable tensions. 
    Thus, the passive output is
    \begin{equation*}
        \dyv = \ActM[][T][(q)] \dot{q} = \begin{carray}{c}
            \displaystyle \sqrt{a^2 + b^2}\sin\rb{ \tan^{-1}\rb{\frac{a}{b}} + \frac{q}{2}} \dot{q} \\ -R \dot{q}
        \end{carray}.
    \end{equation*}
    Since the system is overactuated, there are two possible choices of the actuation coordinates out of singularities of $\ActM[][][(q)]$. When $q = \displaystyle - 2  \tan^{-1}\rb{\frac{a}{b}}$, the first column of $\ActM[][][(q)]$ becomes zero, and the first actuator does not have any effect on the motion.
    In this case, one can easily integrate $\dyv$ as
    \begin{equation*}
        \yv = \gv(q) = \begin{carray}{c}
            \displaystyle -2 \sqrt{a^2 + b^2}\cos\rb{ \tan^{-1}\rb{\frac{a}{b}} + \frac{q}{2}}\\
            -R q
        \end{carray}.
    \end{equation*}
    As expected, the choice of either the first or second component of $\yv$ yields dynamics in the form of~\eqref{eq:coordinate change:overactuated eom}.
    Note that $\yv$ is the tendon displacement with respect to the straight configuration. In fact, the system inputs perform work directly on the finger tendons.
\end{example}
\section{Change of Coordinates\\for Underactuated Systems}~\label{sec:underactuated systems}
We now focus on underactuated systems, specifically dynamics for which the number of independent actuator inputs is less than that of degrees of freedom. In the following, we show that also in this case the actuation coordinates solve the input decoupling problem. First, it is convenient to expand the actuation matrix as
\begin{equation}\label{eq:underactuated:actuation matrix expansion}
    \ActM = \begin{carray}{c}
        \ActM[a]\\
        \ActM[u]
    \end{carray},
\end{equation}
where $\ActM[a] \in \R^{m \times m}$ is nonsingular and $\ActM[u] \in \R^{(n-m) \times m}$. Note that such partition is always possible after reordering the linearly independent rows of $\ActM$. Furthermore, since $n > r = m$, $\dyv$ is an $m$-dimensional vector. 
\begin{theorem}\label{theorem:underactuated}
Assume that~\eqref{eq:lagrangian dynamics} is underactuated, i.e., $m < n$. The integrability assumption is a necessary and sufficient condition for a change of coordinates $\thetav = \hv(\qv) : \mathcal{B}(\qv) \rightarrow \mathcal{N}$ to exist such that~\eqref{eq:lagrangian dynamics} becomes
\begin{equation}\label{eq:coordinate change:underactuated eom}
    \frac{\mathrm{d}}{\mathrm{d}t}\diffT{\mathcal{L}_{\thetav}(\thetav, \dthetav)}{\dthetav} - \diffT{ \mathcal{L}_{\thetav}(\thetav, \dthetav) }{\thetav} = \begin{carray}{c}
        \uv \\ \zerov_{n-m}
    \end{carray},
\end{equation}
i.e., $\ActM[\thetav][][](\thetav)$ has the form
\begin{equation*}
    \ActM[\thetav][][(\thetav)] = \begin{carray}{c}
        \IIm_{n}\\
        \zerom_{(n-m) \times m}
    \end{carray}.
\end{equation*}
Let $\gv(\qv)$ be the integral of $\ActM[][T]\dqv$. Then, a possible choice of $\thetav$ is
    \begin{equation}\label{eq:change coordinates underactuated}
    \begin{split}
        \thetav 
        &= \begin{carray}{c}
            \gv(\qv)\\
            \zerov_{n-m}
        \end{carray} + \begin{carray}{cc}
            \zerom_{m \times m} & \zerom_{m \times (n-m)}\\
            \zerom_{(n - m) \times m} & \IIm_{n-m}
        \end{carray}\qv.
    \end{split}
    \end{equation}
\end{theorem}
\begin{proof}
    The choice of $\thetav$ as given in~\eqref{eq:change coordinates underactuated}
    qualifies as a change of coordinates because its Jacobian
    \begin{equation}\label{eq:jacobian underactuated proof}
    \Jm_{\hv}(\qv) = \begin{carray}{cc}
        \ActM[a][T] & \ActM[u][T]\\
        \zerom_{(n-m) \times m} & \IIm_{n-m}
    \end{carray},
    \end{equation}
    is nonsingular at $\qv$.
    Power invariance and~\eqref{eq:jacobian underactuated proof} imply
    \begin{equation*}
    \begin{split}
        \dthetav^{T}\tauv_{\thetav} = \dqv^{T} \Jm^{T}_{\hv}(\qv) \tauv_{\thetav} = \dqv^{T} \ActM\uv = \dqv^{T}\tauv_{\qv}.
    \end{split}
    \end{equation*}
    Furthermore, being $\dqv$ arbitrary, it follows
    $
    \Jm^{T}_{\hv}(\qv) \tauv_{\thetav} = \ActM \uv,
    $
    which can be rewritten as
    \begin{equation}\label{eq:underactuated proof homogeneous system}
    \underbrace{
    \begin{carray}{cc}
        \ActM[a] & \zerom_{m \times (n - m)}\\
        \ActM[u] & \IIm_{n-m}
    \end{carray}}_{\Jm^{T}_{\hv}(\qv)}
    \begin{carray}{c}
        \tauv_{\thetav_a} - \uv\\
        \tauv_{\thetav_u}
    \end{carray} = \zerov,
    \end{equation}
    where we expanded $\tauv_{\thetav}$ into the two vectors $\tauv_{\thetav_a} \in \R^{m}$ and $\tauv_{\thetav_u} \in \R^{n-m}$ performing work on $\thetav_a$ and $\thetav_u$, respectively. 
    Equation~\eqref{eq:underactuated proof homogeneous system} describes a homogeneous linear system with the unique solution
    \begin{equation}\label{eq:underactuated proof input identity}
    \begin{carray}{c}
        \tauv_{\thetav_a} - \uv\\
        \tauv_{\thetav_u}
    \end{carray} = \zerov,
    \end{equation}
    being that $\Jm^{T}_{\hv}(\qv)$ is nonsingular.

    To prove the necessary part of the statement, suppose there exists $\thetav = \hv(\qv)$ such that~\eqref{eq:coordinate change:underactuated eom} holds and partition
    \begin{equation*}
        \Jm_{\hv}(\qv) = \begin{carray}{c}
            \Jm_{\hv_a}(\qv)\\
            \Jm_{\hv_u}(\qv)
        \end{carray},
    \end{equation*}
    with $\Jm_{\hv_a}(\qv) \in \R^{m \times n}$ and $\Jm_{\hv_u}(\qv) \in \R^{(n-m) \times n}$. Exploiting once again power invariance, we obtain, after some computations,
    \begin{equation*}
        \dqv^{T} \rb{ \Jm^{T}_{\hv_a}(\qv) - \ActM }  \uv = 0,
    \end{equation*}
    or, equivalently, 
    \begin{equation*}
        \Jm_{\hv_a}(\qv) = \ActM[][T].
    \end{equation*}
    Thus, the first $m$ components of $\hv(\qv)$ satisfy the integrability assumption. 
\end{proof}
In Appendix~\ref{appendix:alternative proof}, we report an alternative proof of the sufficient part of Theorem~\ref{theorem:underactuated}, which uses algebraic arguments instead of power invariance. 
\begin{remark}
There is no constraint on choosing the unactuated variables, except that the corresponding Jacobian is nonsingular. Indeed, the factorization given in~\eqref{eq:underactuated proof homogeneous system} holds independently of $\thetav_{u}$. In other words, Theorem~\ref{theorem:underactuated} does not rely on a specific choice of $\thetav_{u}$, which could be used to further simplify the structure of the equations of motion. 
\end{remark}
Note that the previous results can also be derived in a Hamiltonian formulation by considering the type 2 generating function~\cite{goldstein1980classical} $G_{2}(\qv, \tilde{\pv}) = \hv^{T}(\qv)\tilde{\pv}$, where $\tilde{\pv}$ is the momentum in the actuation coordinates.

We illustrate the application of Theorem~\ref{theorem:underactuated} on a soft robotic arm.
\begin{example}[~(Continuum soft robot)]~\label{example: PCC underactuated}
Consider a continuum soft robot discretized into two bodies, modeled under the piecewise constant curvature (PCC) hypothesis. Then, each body has three DOF, corresponding to its curvature $\kappa_{i}$, bending direction $\phi_{i}$ and elongation $\delta L_{i}$, $ i = 1, 2$, so that
\begin{equation*}
    \qv = \begin{carray}{cccccc}
        \kappa_{1} & \phi_{1} & \delta L_{1} & \kappa_{2} & \phi_{2} & \delta L_{2}
    \end{carray}^{T}.
\end{equation*}
Three tendons that run from the base to the tip actuate the robot. Each actuator is located at a distance $d \in \R^{+}$ from the backbone and rotated from the previous by $\SI{120}{\degree}$, as illustrated in Fig.~\ref{fig:example:soft robot underactuated}.
\begin{figure}[!t]
    \centering
    \includegraphics[width = 0.8\percentwidth\columnwidth]{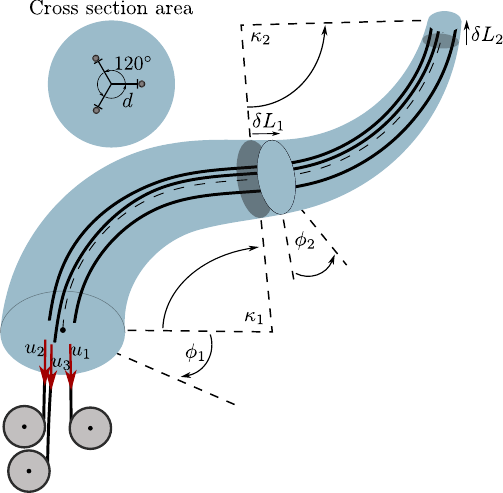}
    \caption{A continuum soft robot discretized into two bodies. Under the piecewise constant curvature assumption, the configuration of each body is described by its curvature $\kappa_{i}$, bending angle $\phi_{i}$ and elongation $\delta L_{i}$ for $i = 1, 2$. 
    Three tendons run from the base to the tip, and spatial motion is obtained by applying a suitable cable tension $\uv \in \R^{3}$. 
    }
    \label{fig:example:soft robot underactuated}
\end{figure}
%
%
%
By exploiting the principle of virtual works, one can show that $\tauv_{\qv} = \ActM\uv$, where
\begin{equation}\label{eq:example PCC:actuation matrix}
    \arraycolsep=2.5pt\def\arraystretch{1.4}
    \ActM = \small \begin{carray}{ccc} -d \mathrm{c}_{2} & d \rb{\frac{1}{2}\mathrm{c}_{{2}} - \frac{\sqrt{3}}{2}\sina{2}}  & d \rb{ \frac{1}{2}\cosa{2} + \frac{\sqrt{3}}{2}\sina{2} } \\ 
    d q_{1} \mathrm{s}_{{2}} & -d q_{1} \rb{ \frac{1}{2} \sina{2} + \frac{\sqrt{3}}{2} \cosa{2} } & -d q_{1} \rb{ \frac{1}{2} \sina{2} - \frac{\sqrt{3}}{2} \cosa{2}}\\ 
    1 & 1 & 1\\ 
    -d \mathrm{c}_{{5}} & d \rb{\frac{1}{2}\mathrm{c}_{{5}} - \frac{\sqrt{3}}{2}\sina{5}} &d \rb{ \frac{1}{2}\cosa{5} + \frac{\sqrt{3}}{2}\sina{5} }\\ d q_{4} \mathrm{s}_{{5}} & -d q_{4} \rb{ \frac{1}{2} \sina{5} + \frac{\sqrt{3}}{2} \cosa{5}} & -d q_{4} \rb{ \frac{1}{2} \sina{5} - \frac{\sqrt{3}}{2} \cosa{5}}\\
    1 & 1 & 1 
    \end{carray},
\end{equation}
and $\uv \in \R^{3}$ collects the cables tension. 
It is easy to verify that $\mathrm{rank}(\ActM) = 3$ except when $q_{1} = q_{4} =~0$ and $q_{2} = q_{5} + k\pi$ with $k \in \mathbb{Z}$. %
In fact, control authority is lost when the arm is in the stretched configuration. However, this is an artifact due to the choice of the bending direction as generalized coordinate~\cite{della2020improved}. Thus, the system is underactuated with $m = 3$ and $n = 6$. 

It can be shown that (derivations are omitted for the sake of space) $\dyv = \ActM[][T]\dqv \in \R^{3}$ can be integrated as
\begin{equation*}\small
    \yv = \begin{carray}{c}
        q_{3} + q_{6} + d( q_{1}\mathrm{c}_{{2}} + q_{4}\mathrm{c}_{{5}})\\
        q_{3} + q_{6} - \frac{d}{2}( q_{1}\mathrm{c}_{{2}} + q_{4}\mathrm{c}_{{5}}) + \frac{\sqrt{3}}{2}d( q_{1}\mathrm{s}_{{2}} + q_{4}\mathrm{s}_{{5}})\\
        q_{3} + q_{6} - \frac{d}{2}( q_{1}\mathrm{c}_{{2}} + q_{4}\mathrm{c}_{{5}}) - \frac{\sqrt{3}}{2}d( q_{1}\mathrm{s}_{{2}} + q_{4}\mathrm{s}_{{5}})
    \end{carray}.
\end{equation*}
According to Theorem~\ref{theorem:underactuated}, the ID problem is solvable through a coordinate change having the form
\begin{equation}\label{eq:example PCC: coordinates change}
    \thetav = \begin{carray}{c}
        \thetav_{a}\\
        \thetav_{u}
    \end{carray} = \begin{carray}{c}
        \yv\\
        \thetav_{u}
    \end{carray},
\end{equation}
where $\thetav_{u} \in \R^{3}$ is any complement to $\yv$. For example, a possible choice is
\begin{equation*}
    \thetav_{u} = \begin{carray}{c}
        q_{3} + d q_{1}\mathrm{c}_{{2}}\\
        q_{3} - \frac{d}{2} q_{1}\mathrm{c}_{{2}} + 
        \frac{\sqrt{3}}{2}d q_{1}\mathrm{s}_{{2}}\\
        q_{3} - \frac{d}{2} q_{1}\mathrm{c}_{{2}} - \frac{\sqrt{3}}{2} d q_{1}\mathrm{s}_{{2}}
    \end{carray}.
\end{equation*}

Remarkably, the components of $\thetav_{a}$ correspond to the change of tendons length inside the arm. Similarly, $\thetav_{u}$ collects the tendons elongation in the first body only. 

The reader can verify that $\Jm^{-T}_{\hv}$ takes the expression given by~\eqref{eq:example PCC:inverse transpose jacobian}, which yields
\begin{equation*}
    \tauv_{\thetav}(\thetav, \uv) = \Jm^{-T}_{\hv}(\qv)\ActM\uv =
    \begin{carray}{c}
        \IIm_{3}\\
        \zerom_{3 \times 3}
    \end{carray}\uv.
\end{equation*}
\end{example}
\begin{figure*}[ht!]
\begin{equation}\label{eq:example PCC:inverse transpose jacobian}
    \Jm^{-T}_{\hv} = 
        \begin{carray}{cccccc} 0 & 0 & 0 & -\frac{2}{3d}\mathrm{c}_{5} & \frac{2}{3dq_{4}}\mathrm{s}_{5} & \frac{1}{3}\\
        0 & 0 & 0 & \frac{2}{3d} \rb{ \frac{1}{2}\cosa{5} - \frac{\sqrt{3}}{2}\sina{5}}  & -\frac{2}{3dq_{4}} \rb{ \frac{1}{2} \sina{5} + \frac{\sqrt{3}}{2} \cosa{5} } & \frac{1}{3}\\
        0 & 0 & 0 & \frac{2}{3d} \rb{ \frac{1}{2}\cosa{5} + \frac{\sqrt{3}}{2}\sina{5}} & -\frac{2}{3dq_{4}} \rb{ \frac{1}{2} \sina{5} - \frac{\sqrt{3}}{2} \cosa{5} } & \frac{1}{3}\\
        -\frac{2}{3d}\mathrm{c}_{{2}} & \frac{2}{3dq_1}\mathrm{s}_{{2}} & \frac{1}{3} & \frac{2}{3d}\mathrm{c}_{{5}} & -\frac{2}{3dq_4}\mathrm{s}_{{5}} & -\frac{1}{3}\\
        \frac{2}{3d} \rb{ \frac{1}{2}\cosa{2} - \frac{\sqrt{3}}{2} \sina{2}} & -\frac{2}{3d} \rb{ \frac{1}{2}\cosa{2} + \frac{\sqrt{3}}{2} \sina{2}} & \frac{1}{3} & -\frac{2}{3d} \rb{ \frac{1}{2}\cosa{5} - \frac{\sqrt{3}}{2}\sina{5}} & \frac{2}{3dq_{4}} \rb{ \frac{1}{2} \sina{5} + \frac{\sqrt{3}}{2} \cosa{5} } & -\frac{1}{3}\\
        \frac{2}{3d}\rb{ \frac{1}{2}\cosa{2} + \frac{\sqrt{3}}{2} \sina{2}} & -\frac{2}{3d}\rb{ \frac{1}{2}\cosa{2} - \frac{\sqrt{3}}{2} \sina{2}} & \frac{1}{3} & -\frac{2}{3d} \rb{ \frac{1}{2}\cosa{5} + \frac{\sqrt{3}}{2}\sina{5}} & \frac{2}{3dq_{4}} \rb{ \frac{1}{2} \sina{5} - \frac{\sqrt{3}}{2} \cosa{5} } & -\frac{1}{3}\\
        \end{carray}.
\end{equation}
\end{figure*}
\section{Integrability of Thread-Like Actuators}~\label{sec:thread-like actuators}
This section shows that the conclusions drawn in Examples~\ref{example:finger} and~\ref{example: PCC underactuated} hold for any mechanical system driven by thread-like actuators. 
This type of actuation is growing in popularity because it allows creating lightweight structures with high power density, and precise and distributed actuation~\cite{palli2011modeling}. We first prove the existence of the actuation coordinates for chains of rigid bodies. We then extend such result to continuum bodies described by reduced-order models. 

Consider a mechanical system of rigid bodies with $n$-DOF actuated through $m$ inelastic tendons. As described in~\cite[Chap.~6]{murray1994mathematical}, one can always define $m$ extension functions $g_{i}(\qv);\ i \in \{ 1, \cdots, m \}$, that measure the tendons displacement as a function of $\qv$. The application of the principle of virtual works yields
\begin{equation*}
    \tauv_{\qv} = \Jm_{\gv}^{T\!}(\qv)\uv = \ActM\uv,
\end{equation*}
where $\uv \in \R^{m}$ collects the tendons tension. It immediately follows that the passive output $\dyv$ is integrable as $\yv = \gv(\qv)$.  

This result extends to mechanical systems with continuum bodies modeled under the Geometric Variable Strain (GVS) technique, see~\cite{renda2020geometric, boyer2020dynamics} for a detailed presentation of all the quantities defined in the following.
We denote the strain as $\xiv\in \R^{6}$, where $X \in [0, L]$ is the curvilinear abscissa with $L$ the body rest length. The GVS approach reduces the infinite-dimensional state of the system by assuming that $\xiv$ admits a representation of the form
\begin{equation}\label{eq:kinematics:functional strain}
    \xiv = \phiv(X, \qv), 
\end{equation}
where $\qv \in \R^{n}$ is the configuration vector, parameterizing the strain. Under~\eqref{eq:kinematics:functional strain}, the dynamic model of a continuum takes the form of~\eqref{eq:lagrangian dynamics}~\cite{armanini2023soft}.
For thread-like actuators, such as tendons and thin fluidic chambers, the generalized actuation force is $\tauv_{\qv} = \ActM\uv$ where 
\begin{equation*}
    \ActM = \int_{0}^{L}\Jm^{T}_{\phiv} \Phim_{a}(X, \qv)\drm X,
\end{equation*}
and $\Phim_{a}(X, \qv) \in \R^{6 \times m}$ is the spatial actuation matrix, whose $i$-th column
\begin{equation}\label{eq:column i:spatial actuation matrix}
    (\Phim_{a}(X, \qv))_{i} = \begin{carray}{c}
        \tildv_{i}(X) \tv_{i}(X, \qv)\\
        \tv_{i}(X, \qv)
    \end{carray} \in \R^{6},
\end{equation}
represents the distributed force of the $i$-th actuator.
In the above expression, $\dv_{i}(X) \in \R^{3}$ and $\tv_{i}(X, \qv) \in \R^{3}$ are the actuator distance to the body backbone and its unit tangent vector~\cite{renda2022geometrically}, respectively. The latter can be computed as
\begin{equation*}
    \tv_{i}(X, \qv) = \frac{[\hxiv \dv_{i} + \dv^{'}_{i}]_{3}}{\norm{\hxiv \dv_{i} + \dv^{'}_{i}}},
\end{equation*}
where $\dv_{i}$ is expressed in homogeneous coordinates and $(\cdot)^{'} := \diff{(\cdot)}{X}$. 
Given $\tv_{i}(X, \qv)$ it is also possible to compute the length $L_{c_{i}}$ of the actuator as
\begin{equation}\label{eq:actuator length}
    L_{c_{i}}(\qv) = \int_{0}^{L} \tv^{T}_{i}(X, \qv) [ \hxiv \dv_{i}(X) + \dv^{'}_{i}(X) ]_{3} \drm X,
\end{equation}
or, after some manipulations, 
\begin{equation*}
    L_{c_{i}}(\qv) = \int_{0}^{L} (\Phim_{a}(X, \qv))_{i}^{T} \rb{ \xiv + \begin{carray}{c}
        \zerov_{3}\\
        \dv^{'}_{i}
    \end{carray} }\drm X.
\end{equation*}
The time derivative of~\eqref{eq:actuator length} is
\begin{equation*}
    \begin{split}
        \Dot{L}_{c_{i}}(\qv) &= \int_{0}^{L} \tv^{T}_{i}[\dot{\Hat{\xiv}} \dv_{i}]_{3} dX = \int_{0}^{L} (\Phim_{a}(X, \qv))^{T}_{i} \Jm_{\phiv}\drm X \dqv\\
        &= (\ActM)^{T}_{i}\dqv,
    \end{split}
\end{equation*}
which implies that the passive output $\dyv = \ActM[][T]\dqv$ is integrable as
\begin{equation*}
    \yv = \gv(\qv) = \begin{carray}{c}
        L_{c_{1}}(\qv)\\
        \vdots\\
        L_{c_{m}}(\qv)
    \end{carray} \in \R^{m}.
\end{equation*}
Thus, the actuation coordinates correspond to the length of the actuators, as for rigid systems.
The above results are independent of the number of DOF and actuators. 
In other words, finite-dimensional models of mechanical systems actuated via tendons always admit a collocated form, independently of being fully-, over-, or underactuated.
Recalling that the actuator coordinates are defined up to a constant, it is also possible to consider the actuator elongation $\delta L_{c_{i}} := L_{c_{i}} - L^{*}_{c_{i}};\ i \in \{ 1, \cdots, m\}$, with respect to a reference length $L^{*}_{c_{i}} \in \R$, such as that in the stress-free configuration. This way, proprioceptive sensors like encoders can easily measure the actuation coordinates. Consequently, the proposed change of coordinates is also helpful for control synthesis.

In Appendix~\ref{appendix:volumetric actuators}, we show that similar arguments apply to soft robots with volumetric actuators.

%
\section{Control of Collocated Underactuated\\Mechanical Systems }~\label{sec:example}
The above results prove a fact empirically observed in soft and continuum robot control. In particular, several works~\cite{jones2006kinematics, braganza2007neural, falkenhahn2016dynamic, li2018design} have shown that it is possible to obtain excellent closed-loop performance in shape and position tasks by controlling the actuator length. This is the case for both model-based and model-free approaches. However, to the best of Authors knowledge, it has never been clarified why these coordinates represent a better choice than others, such as the curvature and bending direction. 
In the actuation coordinates, the dynamics is collocated, which is expected to simplify and robustify the closed loop, especially when the control law does not require significant system knowledge. When a controller is implemented in the actuation coordinates, explicit inversion of the actuation matrix is unnecessary because these coordinates inherently incorporate the inversion. It is also worth noting that any control problem formulated in the initial configuration space can be reformulated in the actuation coordinates.
Remarkably, the above considerations remain true also when the dynamics is underactuated. Furthermore, note that direct inversion of $\ActM$ is not possible in this case being $\ActM$ a tall matrix.
These results allow extending the controllers of~\cite{boyer2020dynamics, pustina2022feedback, borja2022energy, franco2021energy, caasenbrood2021energy, pustina2023psatid} for planar underactuated mechanical systems with damping to those moving in 3D. The following corollary formalizes such statement for the regulators of~\cite{borja2022energy} and~\cite{pustina2023psatid}. 
\begin{corollary}
    Consider an underactuated mechanical system satisfying the same hypotheses of Theorem~\ref{theorem:underactuated}. Suppose that there exists a dissipation function $\cal{F}(\qv, \dqv)$ such that, for all $\dqv \in T_{\qv}\mathcal{M}$,
    \begin{equation}\label{eq:damping assumption}
        \diff{\cal{F}(\qv, \dqv)}{\dqv} \dqv > 0,
    \end{equation}
    and, in the actuation coordinates, 
    \begin{equation}\label{eq:convexity assumption potential energy}
        \diff{^{2} \mathcal{L}_{\thetav}(\thetav, \zerov) }{\thetav^{2}_{u}} > 0.
    \end{equation}
    Let $\Km_{P}, \Km_{D}, \Km_{I}\in \R^{m \times m} > 0$ and $\gamma > 0$. There exist constants $\alpha > 0$ and $\Bar{\gamma} > 0$ such that, if $\Km_{P} > \alpha \IIm_{m}$ and $\gamma > \bar{\gamma}$, both the following regulators
    \begin{enumerate}
        \item PD+ (with feedforward)~\cite{borja2022energy}:
        \begin{equation}\label{eq:PD with feed-forward}
        \begin{split}
            \uv &= -\diffT{ \mathcal{L}_{\thetav}(\thetav_{ad}, \thetav_{ud}, \zerov, \zerov) }{\thetav_{a}}\\
            &\quad+ \Km_{P}( \thetav_{ad} - \thetav_{a} ) - \Km_{D} \dthetav_{a},
        \end{split}
        \end{equation}
        \item P-satI-D~\cite{pustina2023psatid}:
        \begin{equation}\label{eq:P-satI-D}
        \begin{split}
            \uv &= \Km_{P}\rb{ \thetav_{ad} - \thetav_{a} } - \Km_{D}\dthetav_{a}\\
            &\quad+ \frac{\Km_{I}}{\gamma}\int_{0}^{t} \boldsymbol{\mathrm{tanh}} \rb{ \thetav_{ad} - \thetav_{a} }(z)\drm z,
        \end{split}
        \end{equation}
    \end{enumerate}
    will globally asymptotically stabilize the closed-loop system at $(\thetav_{a} \,\, \thetav_{u} \,\, \dthetav_{a}\,\,\dthetav_{u}) = ( \thetav_{ad} \,\, \thetav_{ud} \,\, \zerov \,\, \zerov)$, where $\thetav_{ad} \in \R^{m}$ and $\thetav_{ud} \in \R^{n-m}$ is the unique solution to
    \begin{equation*}
        \diffT{ \mathcal{L}_{\thetav}(\thetav_{ad}, \thetav_{u}, \zerov, \zerov) }{\thetav_{u}} = \zerov.
    \end{equation*}
\end{corollary}
Under~\eqref{eq:convexity assumption potential energy}, for any value of the actuated coordinates $\thetav_a$, there is a unique equilibrium of those unactuated, i.e., the system equilibria are uniquely determined by $\thetav_{a}$. Instead,~\eqref{eq:damping assumption} guarantees internal stability of the closed-loop system.
Note that the above controllers admit more general structures, see~\cite{borja2022energy} and~\cite{pustina2023psatid}. In addition, despite being developed for continuum soft robots, these apply to any underactuated mechanical system with damping on the unactuated variables.

We exploit the previous corollary to perform a shape regulation task for a continuum soft robot moving in 3D.  
The robot has rest length $L = \SI{0.4}{[\meter]}$ and cross section radius $R \in [0.02;\ 0.008]~\si{[\meter]}$, which varies linearly from the base to the tip. The mass density is $\rho = \SI{680}{[\kilogram \per \cubic \meter]}$. Furthermore, we consider a linear visco-elastic stress-strain curve with Young modulus $E = \num{8.88e5} \si{[\newton \per \square \meter]}$, Poisson ratio $P = 0.5$ and material damping $D = \num{1e4}\si{[\newton \per \square \meter \second]}$. Eight tendons actuate the robot. The first six have an oblique routing and are displaced $60^{\circ}$ each. Their initial distance from the center line is $\SI{0.0016}{[\meter]}$. Three of these run from the base to half of the robot, while the remaining ones up to the tip. The last two tendons have an helical routing with pitch $\frac{0.4}{2 \pi}\si{[\meter]}$ and are displaced $180^{\circ}$, with a distance from the backbone of $\SI{0.006}{[\meter]}$. The strain is modeled as 
\begin{equation*}
    \xiv(X, \qv) = \underbrace{ \begin{carray}{cc}
        \Sigmam^{-1}\Phim_{a}(X, \qv^{*}) & \Phim_{u}(X)
    \end{carray}}_{\Phim_{\xiv}(X)} \qv + \xiv^{*},
\end{equation*}
where $\qv \in \R^{15}$, $\qv^{*} = \zerov_{15}$ and $\xiv^{*} = \left(0 \,\, 0 \,\, 1 \,\, 0 \,\, 0 \,\, 0 \right)^{T}$ denote the stress-free configuration and strain, respectively, $\Sigmam(X) \in \R^{6 \times 6}$ is the positive definite body stiffness matrix and the columns of $\Phim_{a} \in \R^{6 \times 8}$ are defined as in~\eqref{eq:column i:spatial actuation matrix}. The strain basis $\Sigmam^{-1}\Phim_{a}(X, \qv^{*})$ has proven to accurately describe the deformations due to the actuation forces~\cite{renda2022geometrically}. Instead, 
\begin{equation*}\small
\Phim_{u}(X) = \!\!\!\begin{carray}{ccccccc}
    1 & P_{1}(X) & P_{2}(X) & 0 & 0 & 0 & 0\\
    0 & 0 & 0  & 1 & P_{1}(X) & P_{2}(X) & 0\\
    0 & 0 & 0 & 0 & 0 & 0 & 1\\
    \multicolumn{7}{c}{\zerov_{3 \times 7}}
\end{carray},
\end{equation*}
with $P_{1}(X) := \displaystyle 2\frac{X}{L}-1$ and $P_{2}(X) := \displaystyle 6\rb{\frac{X}{L}}^2-6\frac{X}{L}+1$, encodes three Legendre polynomials modeling the angular deformations due to the gravitational field, not captured by $\Sigmam^{-1}\Phim_{a}$. Since $n = 15$ and $r = m = 8$, the system is underactuated. Note that only shape regulation tasks can be achieved in general. We compare~\eqref{eq:P-satI-D} with the PD+ regulator in $\qv$ space of~\cite{della2021model} 
\begin{equation}\label{eq:PD with feedforward q}
\begin{split}
    \uv &= -\ActM[][\dagger][(\qv_{d})]\diffT{ \mathcal{L}_{\qv}(\qv_{d}, \zerov) }{\qv}\\
    &\quad+ \ActM[][T]\left[ k_{\Prm}\rb{ \qv_{d} - \qv } - k_{\Drm}\dqv \right],
\end{split}
\end{equation}
Due to $\ActM$, the above control law guarantees only local asymptotic stability~\cite{della2021model, ortega2013passivity}, and it requires information of the entire state of the robot to be implemented. The control gains of~\eqref{eq:P-satI-D} and~\eqref{eq:PD with feedforward q} are $\Km_{P} = k_{P}\IIm_{8}$, $\Km_{D} = k_{D}\IIm_{8}$, $\Km_{I} = k_{I}\IIm_{8}$ and $\gamma = 1$, with $k_{P} = 2.5 \times 10^{3}~\si{[\newton \per \meter]}$, $k_{D} =  10~\si{[\newton \second \per \meter]}$ and $k_{I} = 2 \times 10^{3}~\si{[\newton \per \meter \second]}$. %
Because of the underactuation, only the configurations satisfying the equilibrium equation
\begin{equation*}\label{eq:equilibrium equation}
    \diffT{ \mathcal{L}_{\qv}(\qv_{eq}, \zerov) }{\qv} = \ActM[][][(\qv_{eq})] \uv,
\end{equation*}
with $\uv \in \R^{8}$, can be controlled. We command the three desired shapes given in~\eqref{eq:simulation1:reference} as step references spaced in time by $\SI{2}{[\second]}$.
\begin{figure*}[!t]
\begin{equation}\label{eq:simulation1:reference}\small
\begin{array}{l}
    \qv_{d, 1} = \begin{carray}{ccccccccccccccc}
        \!\!\!-8.85 &\!\!\! \,\,\,-4.70 &\!\!\! \,\,\, -1.39 &\!\!\!\,\,\, -26.30 &\!\!\! -23.41 &\!\!\! -26.33 &\!\!\! -1.74 &\!\!\! -1.79 &\!\!\! 0.19 &\!\!\!\,\,\,\,\, -0.08 &\!\!\! -0.09 &\!\!\! 0.70 &\!\!\!\,\,\,\, -0.68 &\!\!\! -0.05 &\!\!\! -1.55\!\!\!
    \end{carray},\\
    \qv_{d, 2} = \begin{carray}{ccccccccccccccc}
        \!\!\!-13.21 & \!\!\!-11.74 &\!\!\!
   -9.91 &\!\!\! \,\,\,
  -15.40 &\!\!\!
  -14.20 &\!\!\!
  -14.25 &\!\!\!
   -0.26 &\!\!\!
   -0.72 &\!\!\!
    0.09 &\!\!\!\,\,\,\,\,
   -0.35 &\!\!\!
   -0.11 &\!\!\!
    0.20 &\!\!\!\,\,\,\,
   -2.08 &\!\!\!
   -0.75 &\!\!\!
   -1.75\!\!\!
    \end{carray},\\
    \qv_{d, 3} = \begin{carray}{ccccccccccccccc}
          \!\!\! -12.99 &\!\!\!
  -11.52 &\!\!\!
  -20.49 &\!\!\!
  -22.94 &\!\!\!
  -19.77 &\!\!\!
  -20.40 &\!\!\!
   -0.86 &\!\!\!
   -0.87 &\!\!\!
   -0.55 &\!\!\!
   -0.50 &\!\!\!
   -0.43 &\!\!\!
   -2.01 &\!\!\!
   -1.33 &\!\!\!
    0.10 &\!\!\!\,\,\,\,
    1.51 \,\,\,\,
    \end{carray}.
\end{array}
\end{equation}
\end{figure*}
Furthermore, $\qv_{d, i}$ is converted into a desired tendon displacement $\thetav_{ad,i}(t) = \yv_{d, i}(t)$; $ i = 1, 2, 3$, for~\eqref{eq:P-satI-D}. 
The robot starts from the straight (stress-free) configuration at rest, and the simulation runs for $\SI{6}{[\second]}$. In the following simulations, $\yv$ has been computed through numerical integration because it was impossible to derive its closed form expression. On the other hand, in a experimental setup equipped with motor encoders $\yv$ could have been directly evaluated or obtained from the available measurements.
Figures~\ref{fig:simulation 1:y1}--\ref{fig:simulation 1:y3} and~\ref{fig:simulation 1 dummy:q1}--\ref{fig:simulation 1 dummy:q3} 
show the evolution of the actuation coordinates and the configuration variables under~\eqref{eq:P-satI-D} and~\eqref{eq:PD with feedforward q}, respectively, for three sub-intervals of length $1~[\si{\second}]$. As expected, the P-satI-D regulates the actuation coordinates to the desired set point. On the other hand, the PD+ in $\qv$ space fails this task, always showing a steady state error. However, the closed-loop system remains stable. The control action for the two closed-loop systems is reported in Fig.~\ref{fig:control action}. Note that the controllers outputs are quite different. 
%
Finally, Fig.~\ref{fig:simulation 1:strobo} presents a photo sequence of the two closed-loop systems. The end-effector reaches the correct position only under the P-satI-D. Indeed, the average norm of the steady-state Cartesian error is $5.8\cdot 10^{-5}[\si{\meter}]$ for the P-satI-D and, respectively, $1.2\cdot 10^{-2}[\si{\meter}]$ for the PD+ regulator.
\begin{figure}
    \centering
    \subfigure[{}]{
        \includegraphics[width = 0.9\columnwidth, trim={14cm 19cm 14cm 0cm}, clip]{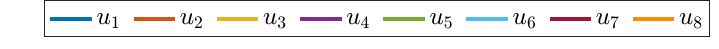}
    }\\\setcounter{subfigure}{0}\vspace{-0.7cm}
    \subfigure[{P-satI-D in $\yv$ space}]{
        \includegraphics[width=0.9\columnwidth]{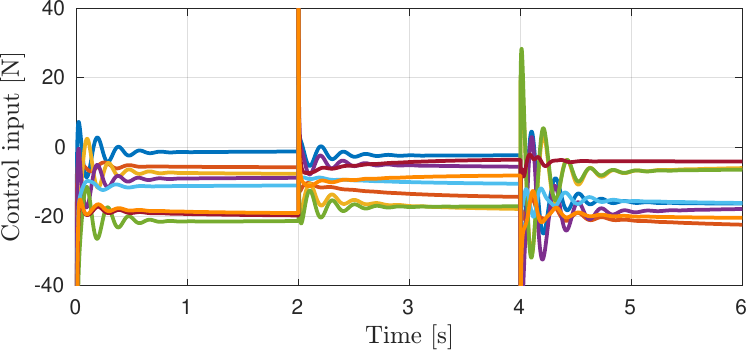}
    }\\
    \subfigure[{PD+ in $\qv$ space}]{
        \includegraphics[width=0.9\columnwidth]{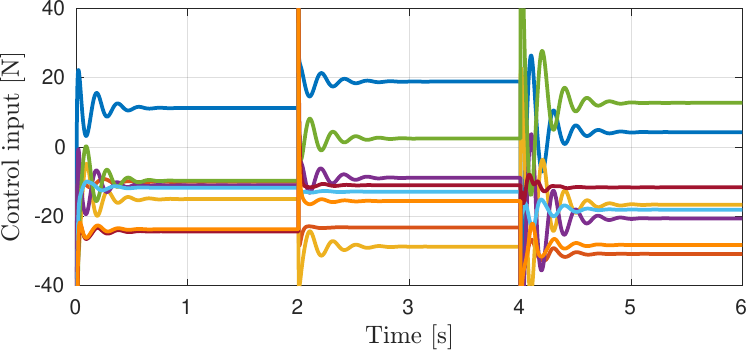}
    }
    \caption{Time evolution of the control inputs under the P-satI-D~\eqref{eq:P-satI-D} in $\yv$ space and the PD+~\eqref{eq:PD with feedforward q} in $\qv$ space, respectively.}
    \label{fig:control action}
\end{figure}
\begin{figure*}[!ht]
    \centering
    \subfigure[{}]{
        \includegraphics[width = \textwidth, trim={0cm 19cm 0cm 0cm}, clip]{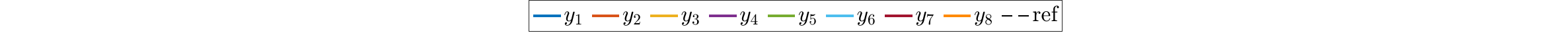}
    }\\\setcounter{subfigure}{0}\vspace{-0.7cm}
    \subfigure[{P-satI-D in $\yv$ space; $t \in [0; 1]~[\si{\second}]$}]{
        \includegraphics[width = 0.31\textwidth]{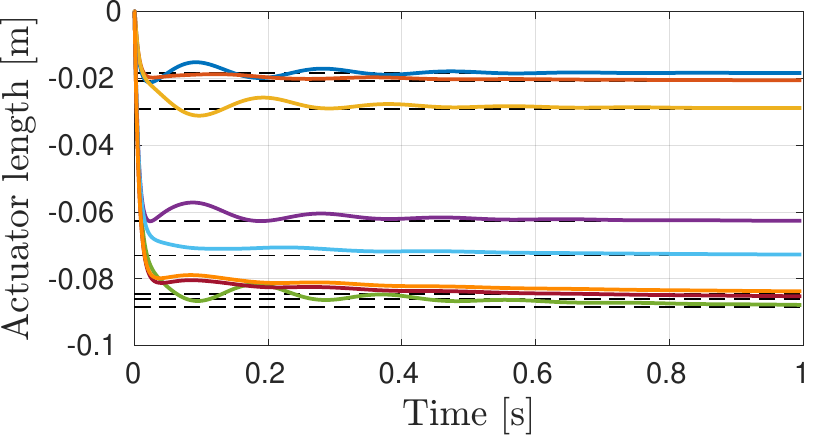}
        \label{fig:simulation 1:y1}
    }\hfill
    \subfigure[{P-satI-D in $\yv$ space; $t \in [2; 3]~[\si{\second}]$}]{
        \includegraphics[width = 0.31\textwidth]{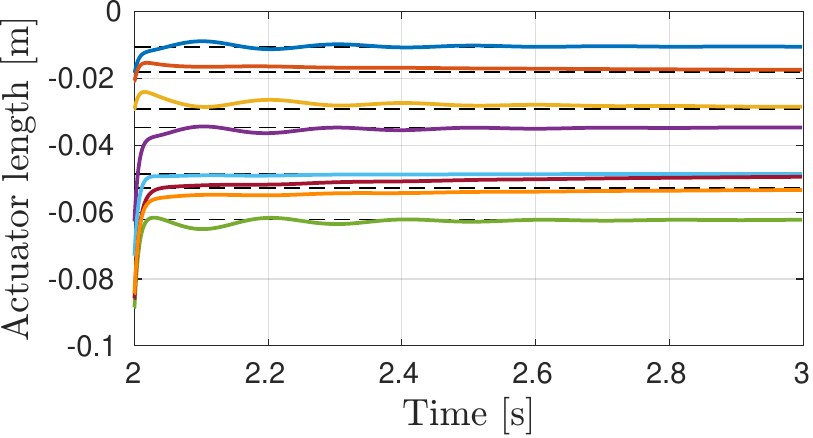}
        \label{fig:simulation 1:y2}
    }\hfill
    \subfigure[{P-satI-D in $\yv$ space; $t \in [4; 5]~[\si{\second}]$}]{
        \includegraphics[width = 0.31\textwidth]{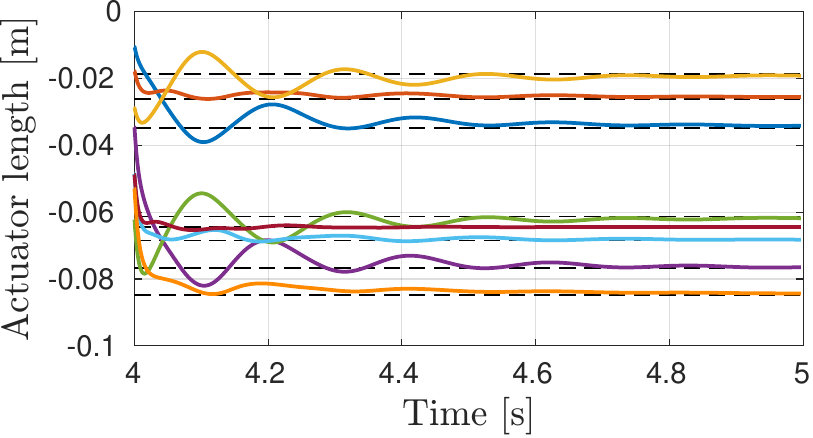}
        \label{fig:simulation 1:y3}
    }\hfill
    \subfigure[{}]{
        \includegraphics[width = \textwidth, trim={0cm 19cm 0cm 0cm}, clip]{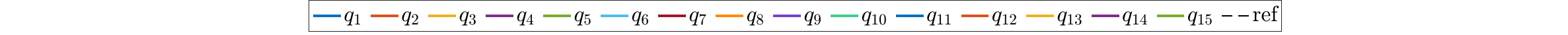}
    }\\\setcounter{subfigure}{0}\vspace{-0.7cm}
    \subfigure[{PD+ in $\qv$ space; $t \in [0; 1]~[\si{\second}]$}]{
        \includegraphics[width = 0.31\textwidth]{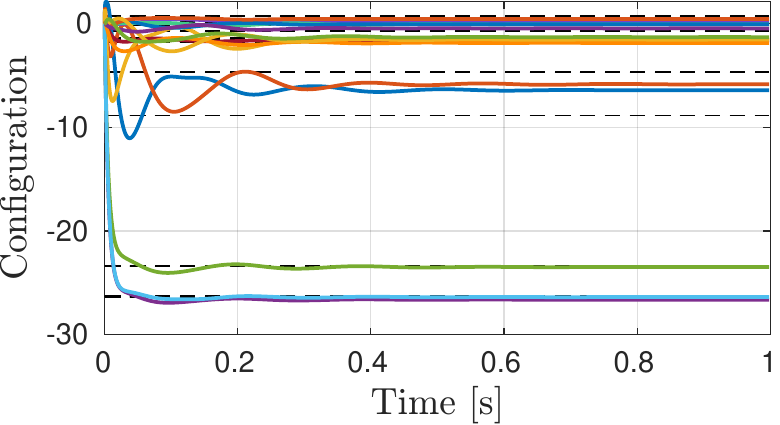}
        \label{fig:simulation 1 dummy:q1}
    }\hfill
    \subfigure[{PD+ in $\qv$ space; $t \in [2; 3]~[\si{\second}]$}]{
        \includegraphics[width = 0.31\textwidth]{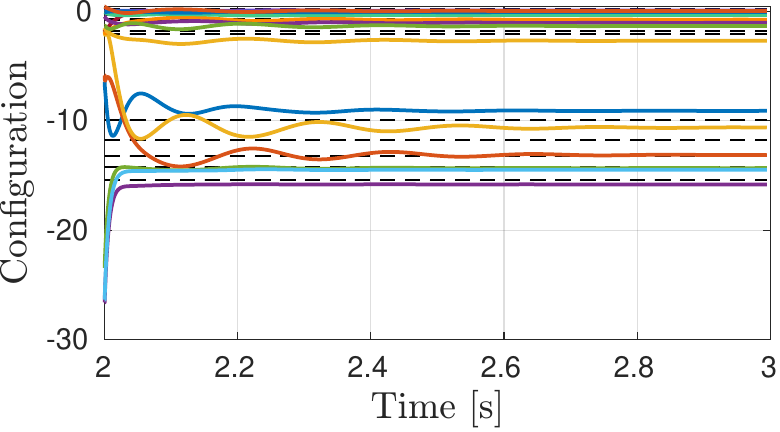}
        \label{fig:simulation 1 dummy:q2}
    }\hfill
    \subfigure[{PD+ in $\qv$ space; $t \in [4; 5]~[\si{\second}]$}]{
        \includegraphics[width = 0.31\textwidth]{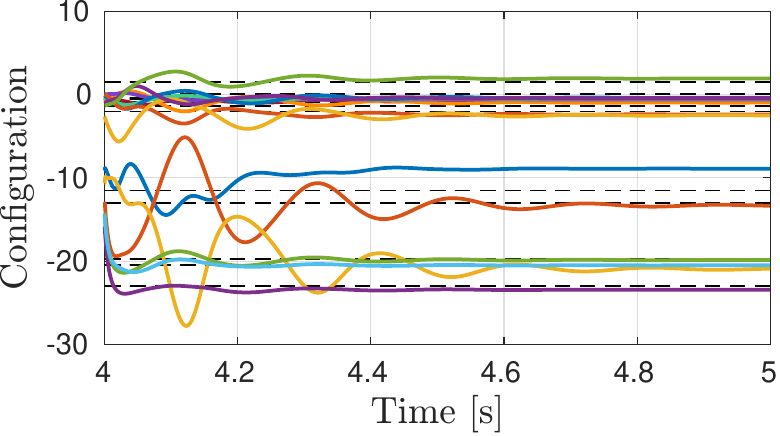}
        \label{fig:simulation 1 dummy:q3}
    }
    \caption{Time evolutions of (a)--(c) actuator elongations and (d)--(f) configuration variables under the P-satI-D~\eqref{eq:P-satI-D} in $\yv$ space and the PD+~\eqref{eq:PD with feedforward q} in $\qv$ space, respectively. The P-satI-D regulates the actuation coordinates to the desired target. On the other hand, the closed-loop system under the PD+ regulator is stable but has a steady-state error.}
    \label{fig:simulation 1:results}
\end{figure*}
%
\begin{figure*}[t!]
    \centering
    \subfigure[{P-satI-D in $\yv$ space; $t \in [0;\ 2]~\si{[\second]}$}]{
        \includegraphics[height = 110px]
        {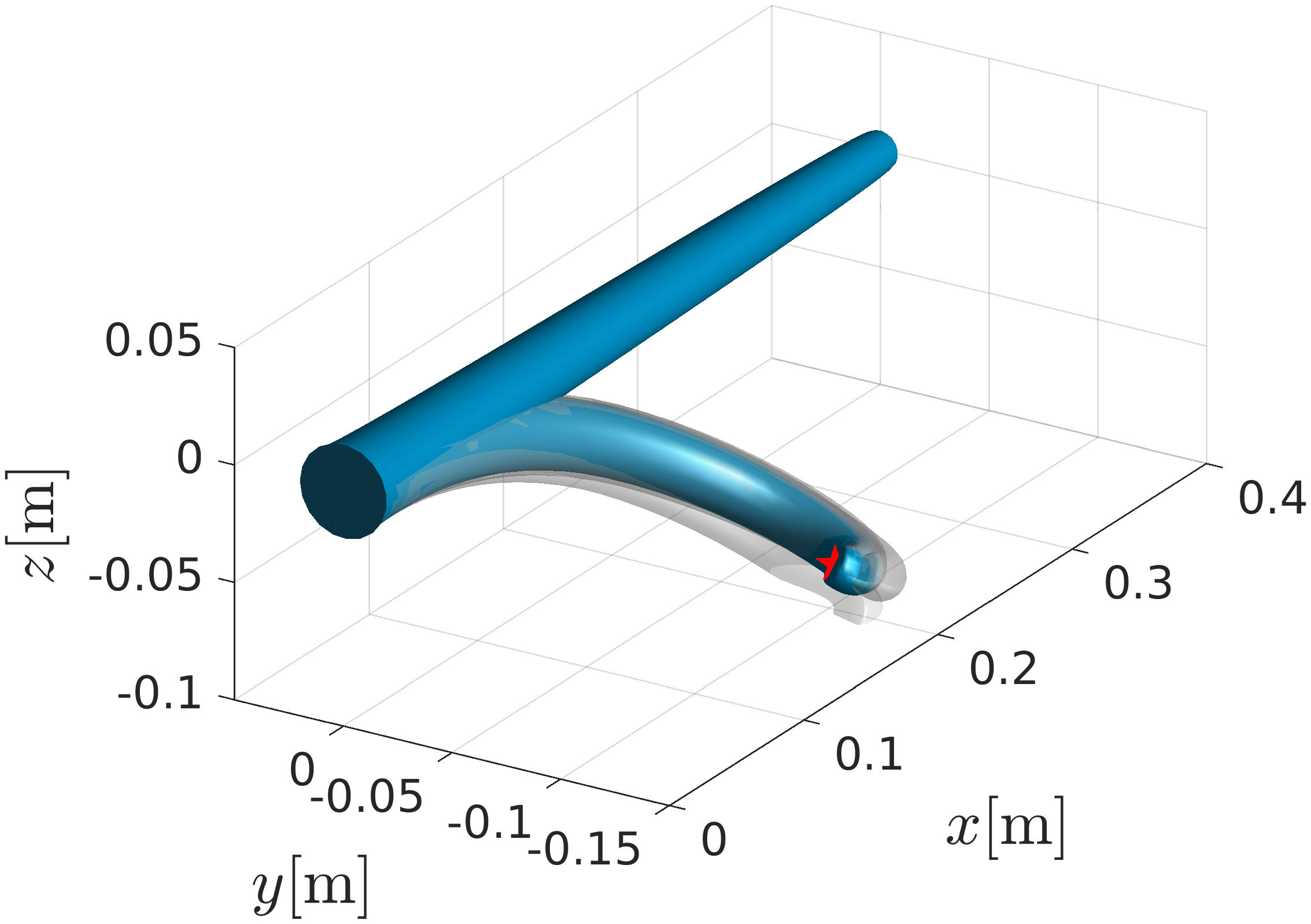}
        \label{fig:simulation 1:strobo 1}
    }\hfill
    \subfigure[{P-satI-D in $\yv$ space; $t \in [2;\ 4]~\si{[\second]}$}]{
        \includegraphics[height = 110px]{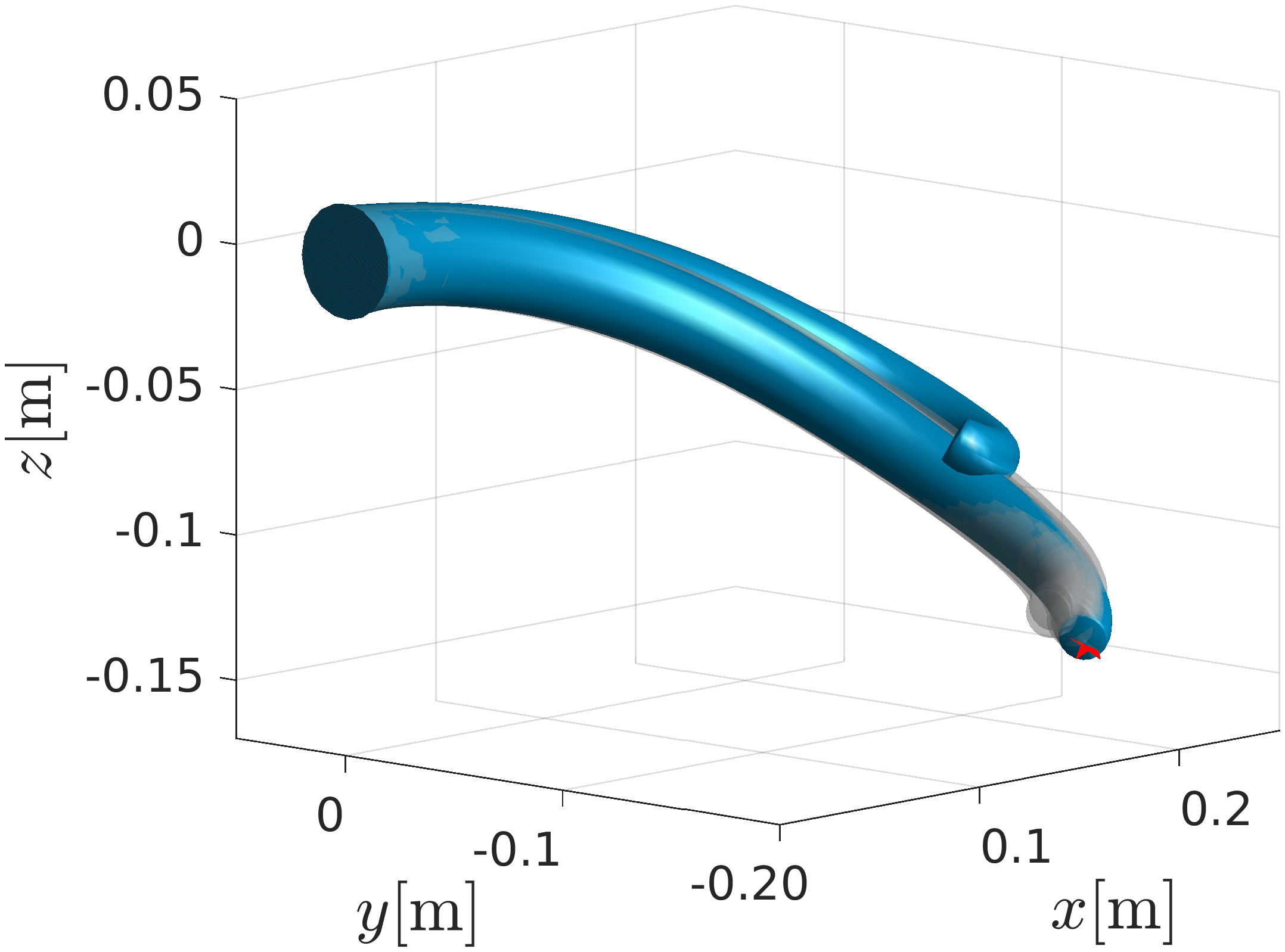}
        \label{fig:simulation 1:strobo 2}
    }\hfill
    \subfigure[{P-satI-D in $\yv$ space; $t \in [4;\ 6]~\si{[\second]}$}]{
        \includegraphics[height = 110px]{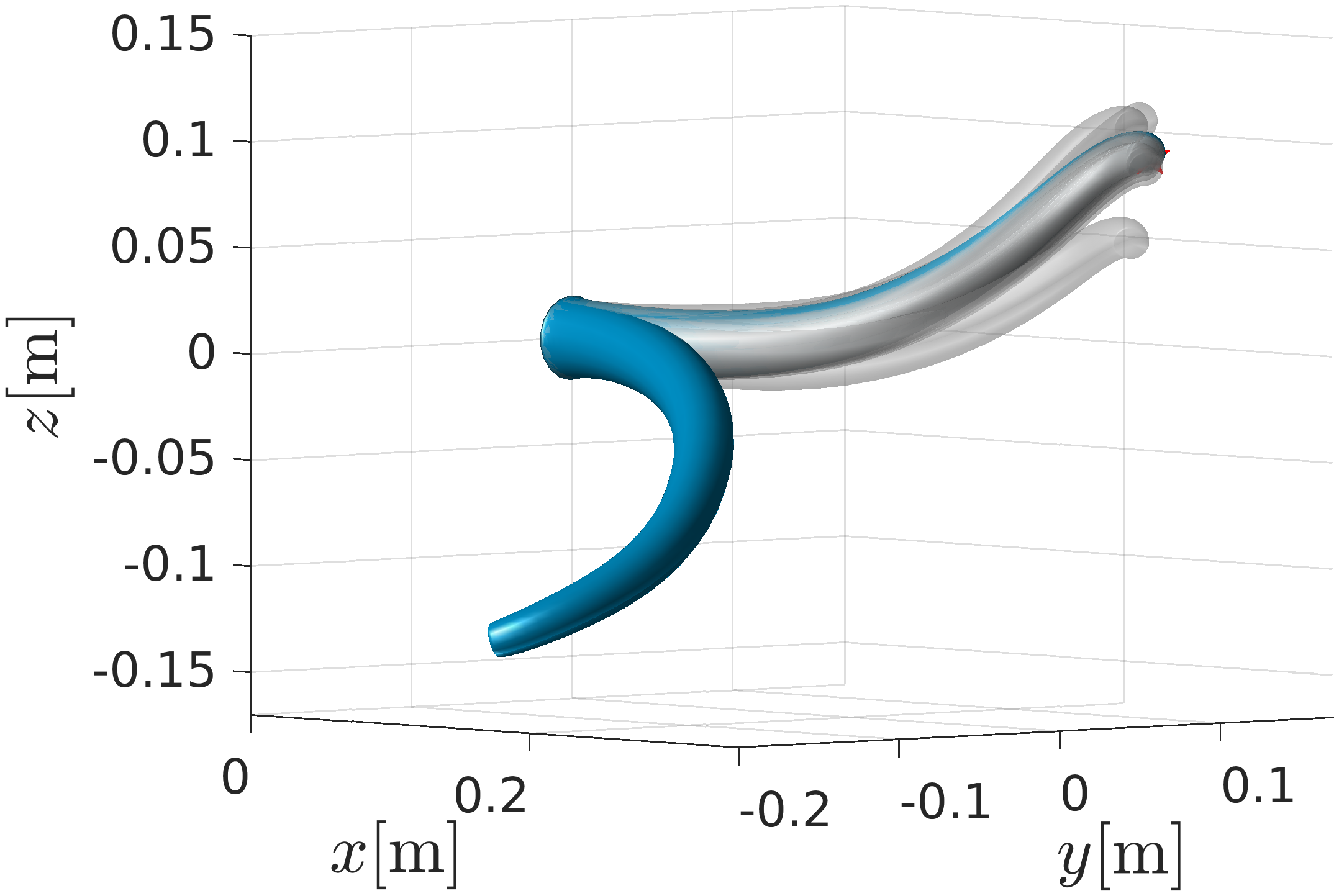}
        \label{fig:simulation 1:strobo 3}
    }
    \subfigure[{PD+ in $\qv$ space; $t \in [0;\ 2]~\si{[\second]}$}]{
        \includegraphics[height = 110px]{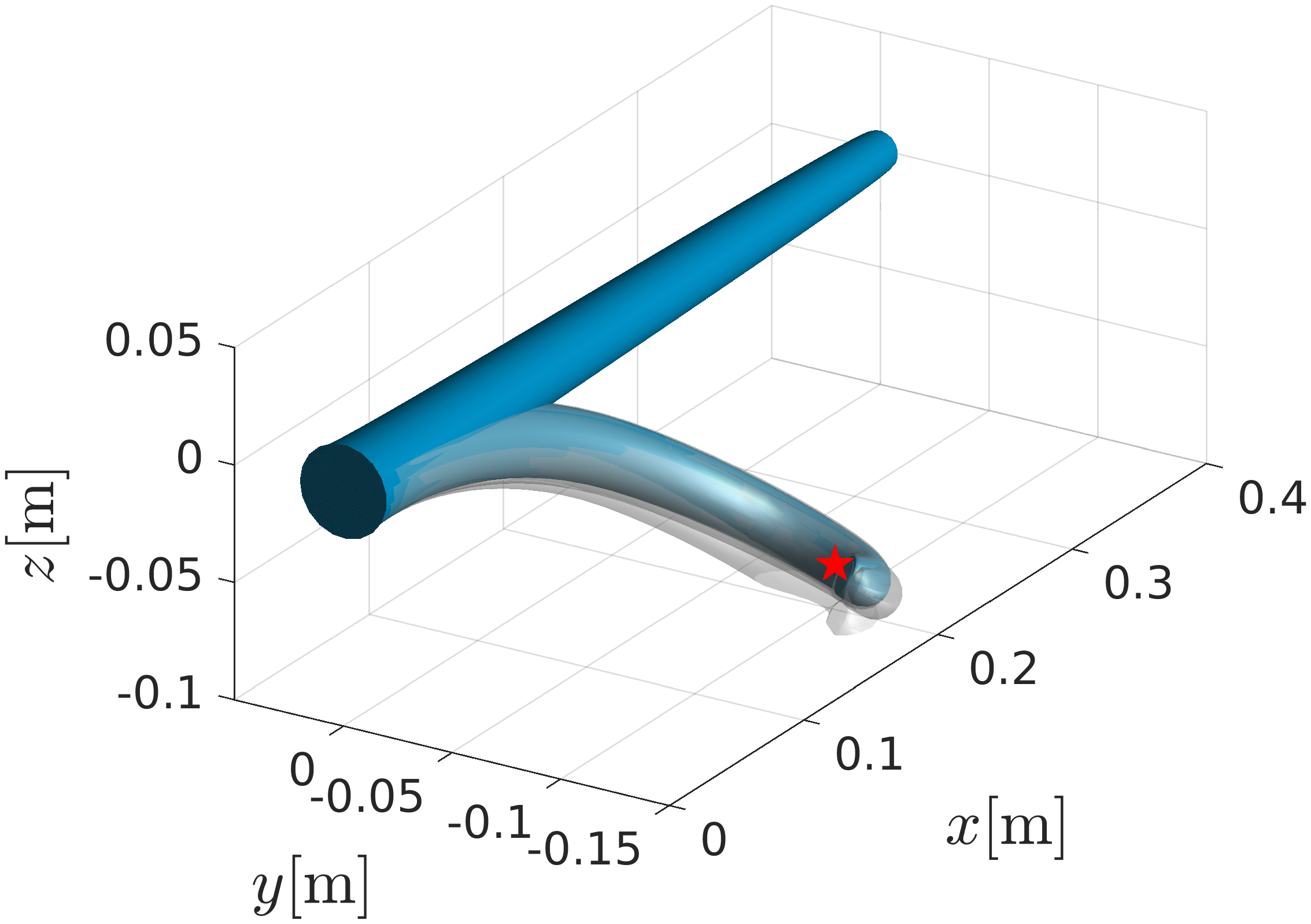}
        \label{fig:simulation dummy 1:strobo 1}
    }\hfill
    \subfigure[{PD+ in $\qv$ space; $t \in [2;\ 4]~\si{[\second]}$}]{
        \includegraphics[height = 110px]{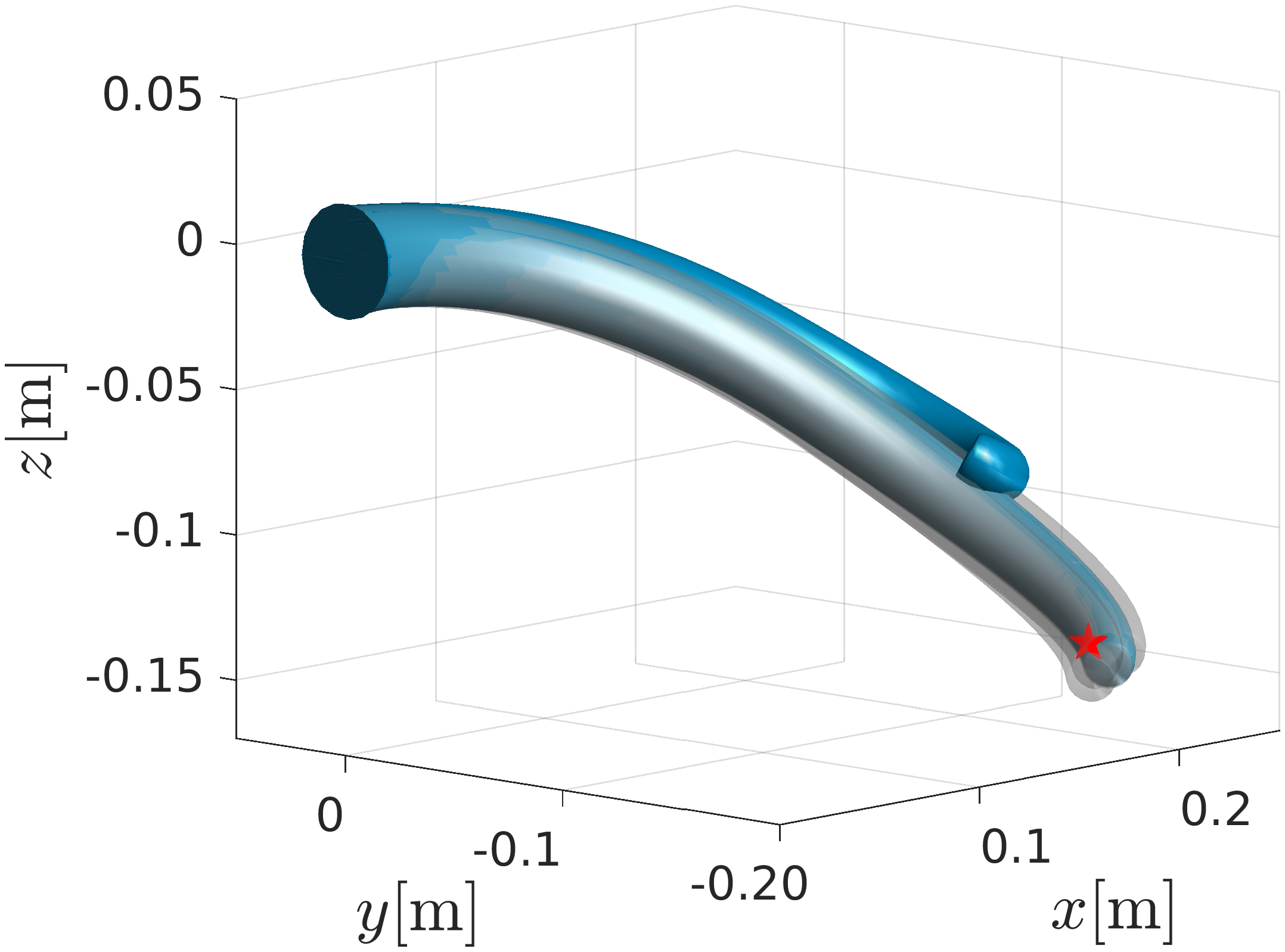}
        \label{fig:simulation dummy 1:strobo 2}
    }\hfill
    \subfigure[{PD+ in $\qv$ space; $t \in [4;\ 6]~\si{[\second]}$}]{
        \includegraphics[height = 110px]{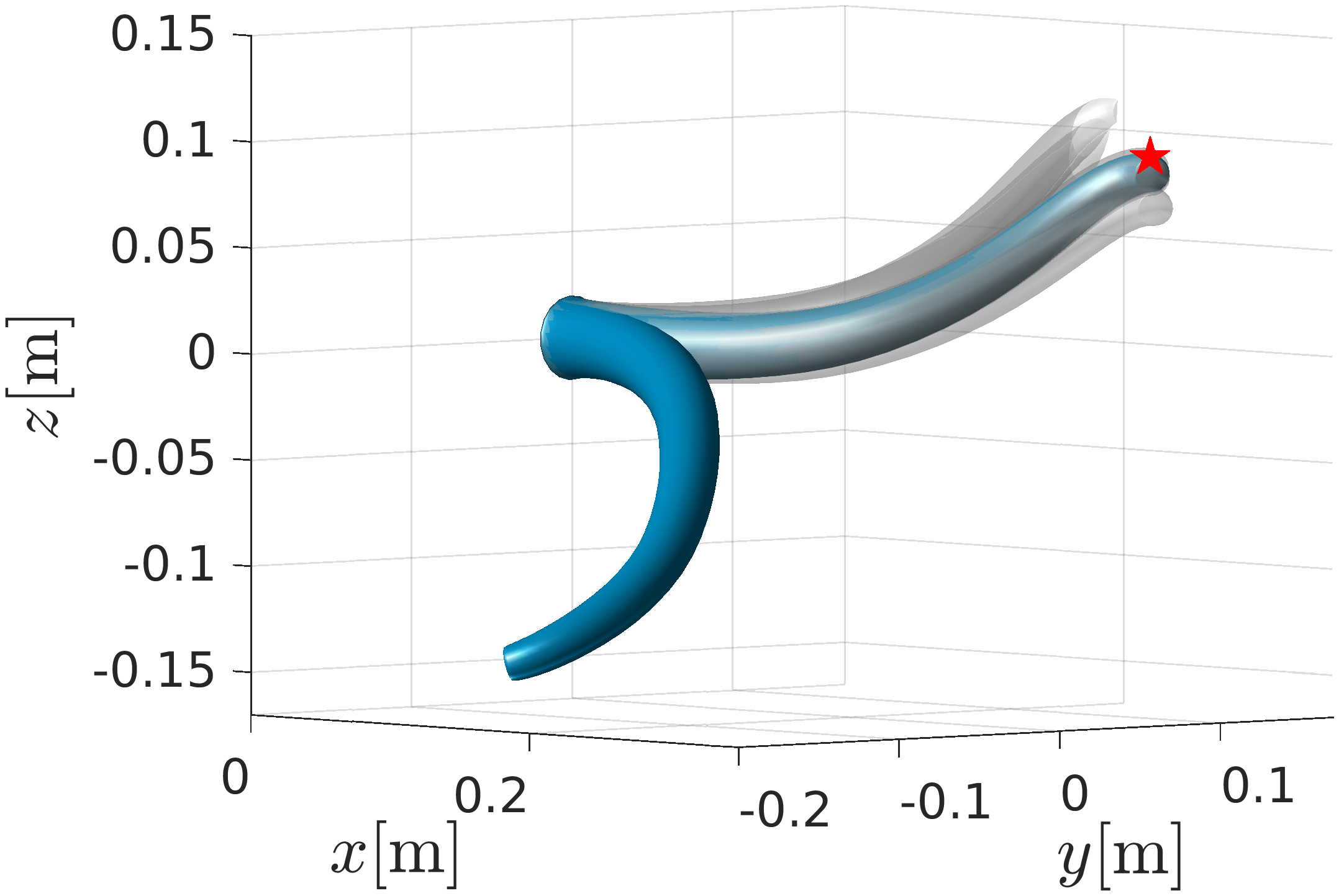}
        \label{fig:simulation dummy 1:strobo 3}
    }
    \caption{Frame sequences of robot motion, divided in three time windows. Figs.~(a)--(c) and (d)--(f) show the robot when controlled in the $\yv$ or $\qv$ coordinates using the laws~\eqref{eq:P-satI-D} and~\eqref{eq:PD with feedforward q}, respectively. 
    The initial and final configurations of each interval are shown in blue, while light gray shapes represent intermediate configurations. A red star indicates the constant target position of the end-effector.}
    \label{fig:simulation 1:strobo}
\end{figure*}
\section{Conclusions}~\label{sec:conclusions}
This article has considered the input decoupling problem for Lagrangian systems. We have shown that there exists a class of Lagrangian dynamics, called \textit{collocated}, for which a coordinate transformation decouples actuator inputs entering the equations of motion through a configuration-dependent actuation matrix. These coordinates have a physical interpretation and can be easily computed. 
Under mild conditions on the differentiability of the actuation matrix, a simple test allows verifying if the dynamics is collocated or not. 
As a consequence of power invariance, the results equally apply to fully actuated, overactuated and underactuated systems. In case of underactuated dynamics, inputs are collocated with the actuation coordinates, while some freedom in left in the definition of the unactuated coordinates.
Since we consider only coordinate transformations, the proposed method differs from the standard differential geometric approach used for controlling input-affine nonlinear systems, which typically makes use of complete feedback transformations.
As a byproduct of our approach, we have shown that all mechanical systems driven by thread-like actuators are collocated. Moreover, we were able to extend control laws recently developed for underactuated systems with constant actuation matrix to collocated mechanical systems with damping.

Future work will be devoted to the experimental validation of the proposed method and to special choices of the unactuated variables that further simplify the equations of motion, ease the check of conditions for obtaining input-state or input-output exact linearization via feedback, or even reveal the existence of flat outputs for the system. Additionally, we will consider relaxing the integrability hypothesis at the cost of transforming the input.

\appendix
\subsection{Properties of Lagrangian systems}\label{appendix:properties}
We recall two important properties of Lagrangian systems that play a key role in deriving the results in this paper. 
\begin{property}\label{property:power}
Let $\mathcal{H}_{\qv}(\qv, \dqv) := \diff{\mathcal{L}_{\qv}(\qv, \dqv)}{\dqv} \dqv - \mathcal{L}_{\qv}(\qv, \dqv)$ be the system Hamiltonian. For all $\qv \in \mathcal{M}, \dqv \in T_{\qv}\mathcal{M}$ and $\uv \in \R^{m}$, it holds
\begin{equation}\label{eq:input power}
    \Dot{\mathcal{H}}_{\qv}(\qv, \dqv) = \dqv^{T}\tauv_{\qv}(\qv, \uv).
\end{equation}
\end{property}
Equation~\eqref{eq:input power} states that the time rate of change of the Hamiltonian, i.e., the system total energy, equals the input power. It also follows from~\eqref{eq:lagrangian dynamics} and~\eqref{eq:input power} that the dynamics is passive with respect to the pair $(\uv, \dyv) = \rb{\uv, \ActM[][T] \dqv}$ with storage function $\mathcal{H}_{\qv}(\qv, \dqv)$.
\begin{property}\label{property:change of coordinates}
    If $\thetav : \mathcal{B}(\qv) \rightarrow \mathcal{N} = \hv(\qv)$ is a (local) diffeomorphism, with $\Jm_{\hv}(\qv) = \displaystyle \diff{\hv}{\qv}$, then
    \begin{align}\label{eq:euler lagrange eqns invariance}
         \frac{\mathrm{d}}{\mathrm{d}t}\diffT{\mathcal{L}_{\thetav}(\thetav, \dthetav)}{\dthetav} - \diffT{ \mathcal{L}_{\thetav}(\thetav, \dthetav) }{\thetav} = \tauv_{\thetav}(\thetav, \uv),
    \end{align}
    where 
    $$\mathcal{L}_{\thetav}(\thetav, \dthetav) = \mathcal{L}_{\qv}(\qv = \hv^{-1}(\thetav), \dqv = \Jm^{-1}_{\hv} \dthetav),$$
    and
    $$\tauv_{\thetav}(\thetav, \uv) = \Jm^{-T}_{\hv}\tauv_{\qv}(\qv = \hv^{-1}(\thetav), \uv).$$ 
    This also implies that, for all $\thetav \in \mathcal{N}, \dthetav \in T_{\thetav}\mathcal{N}, \qv \in \mathcal{M}, \dqv \in T_{\qv}\mathcal{M}$ and $\uv \in \R^{m}$, it holds
    \begin{equation}\label{eq:power invariance}
        \Dot{\mathcal{H}}_{\thetav}(\thetav, \dthetav) = \dthetav^{T}\tauv_{\thetav}(\thetav, \uv) = \dqv^{T}\tauv_{\qv}(\qv, \uv) = \Dot{\mathcal{H}}_{\qv}(\qv, \dqv),
    \end{equation}
    being $\mathcal{H}_{\thetav}(\thetav, \dthetav) = \mathcal{H}_{\qv}(\qv = \hv^{-1}(\thetav), \dqv = \Jm^{-1}_{\hv} \dthetav)$.
\end{property}
According to the above property, the Euler-Lagrange equations and the power are invariant, i.e., they do not depend on the choice of coordinates representing the dynamics. 
\subsection{Alternative proof of Theorem~\ref{theorem:underactuated}}\label{appendix:alternative proof}
We provide an alternative proof of the \textit{if} part of Theorem~\ref{theorem:underactuated}. Similar considerations also hold for Theorem~\ref{theorem:fully actuated} and Corollary~\ref{corollary:overactuated}.  
\begin{proof}
From Property~\ref{property:change of coordinates}, we have
\begin{equation}\label{eq:tau theta identity 2}
    \tauv_{\thetav} = \Jm^{-T}_{\hv}\ActM[][][(\qv = \hv^{-1}(\thetav))]\uv.
\end{equation}
By exploiting the block triangular structure of $\Jm_{\hv}(\qv)$ it follows
\begin{equation*}
    \Jm^{-T}_{\hv}(\qv) = \begin{carray}{cc}
        \ActM[a][-1] & \zerom_{m \times (n-m)}\\
        -\ActM[u][][]\ActM[a][-1] & \IIm_{n-m}
    \end{carray},
\end{equation*}
which yields
\begin{equation*}
\begin{split}
    \Jm^{-T}_{\hv}\ActM[][][] &= \begin{carray}{cc}
        \ActM[a][-1] & \zerom_{m \times (n-m)}\\
        -\ActM[u][][]\ActM[a][-1] & \IIm_{n-m}
    \end{carray}\begin{carray}{c}
        \ActM[a][]\\\ActM[u][]
    \end{carray}\\
    &= \begin{carray}{c}
        \IIm_{m}\\
        \zerom_{(n-m) \times m}
    \end{carray}.
\end{split}
\end{equation*}
\end{proof}
\subsection{Integrability of volumetric actuators}\label{appendix:volumetric actuators}
Following arguments similar to those of~\cite{stolzle2021piston}, it is possible to extend the results of Sec.~\ref{sec:thread-like actuators} to robotic systems with volumetric actuators~\cite{su2022pneumatic}. 

Let $V_{i}^{*}$ be the volume of the $i$-th actuator inside the robot when in the reference configuration. Similarly, denote with $V_{i}(\qv)$ the volume in the current deformed configuration. The work performed by the actuator on the robot is 
\begin{equation*}
    W_{u_i} = \rb{V_{i}(\qv) - V_{i}^{*}}u_{i}.
\end{equation*}
To determine the effect of $u_{i}$ on the generalized coordinates it is possible to apply the principle of virtual works
\begin{equation*}
    \delta W_{(\tauv_{\qv})_{i}} = \delta \qv^{T} (\tau_{\qv})_{i} = \delta \qv^{T} \rb{\frac{\partial V_{i}(\qv)}{\partial \qv}}^{T} u_{i} = \delta W_{u_i},
\end{equation*}
obtaining
\begin{equation*}
    \ActM[i][] = \rb{\frac{\partial V_{i}(\qv)}{\partial \qv}}^{T}. 
\end{equation*}
The above equations imply that $\dyv_{i} = \ActM[i][T]\dqv$ is integrable and the corresponding actuation coordinate can be chosen as $\delta V_{i} := V_{i}(\qv) - V_{i}^{*}$, which is the volume variation in the actuator chamber.  
\bibliographystyle{IEEEtran}
\bibliography{./references}
\vskip -2\baselineskip plus -1fil
\begin{IEEEbiography}[{\includegraphics[width=1in,height=1.25in,clip,keepaspectratio]{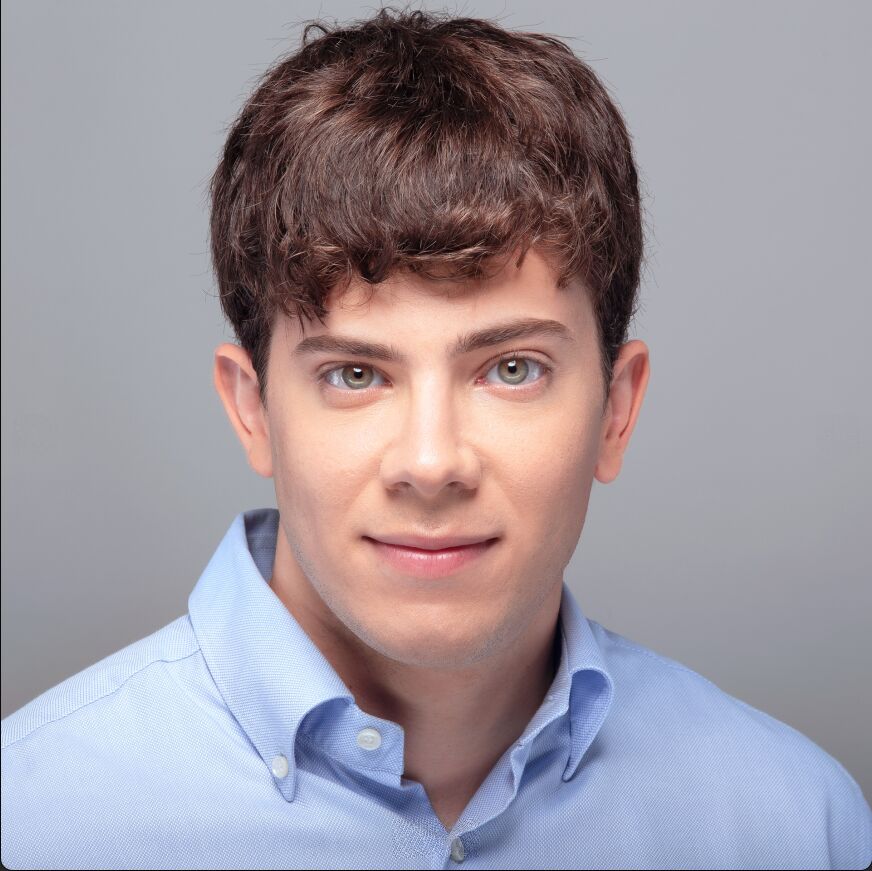}}]{Pietro Pustina} (Student Member, IEEE) 
received a M.~Sc.~in Control Engineering from Sapienza University of Rome, Italy, in 2021.
Since November 2021, he has been working towards a Ph.D. in Automatic Control at Sapienza University of Rome, Italy. His research interests include modeling and control of continuum soft robots. 
\end{IEEEbiography}
\vskip -2\baselineskip plus -1fil
\begin{IEEEbiography}[{\includegraphics[width=1in,height=1.25in,clip,keepaspectratio]{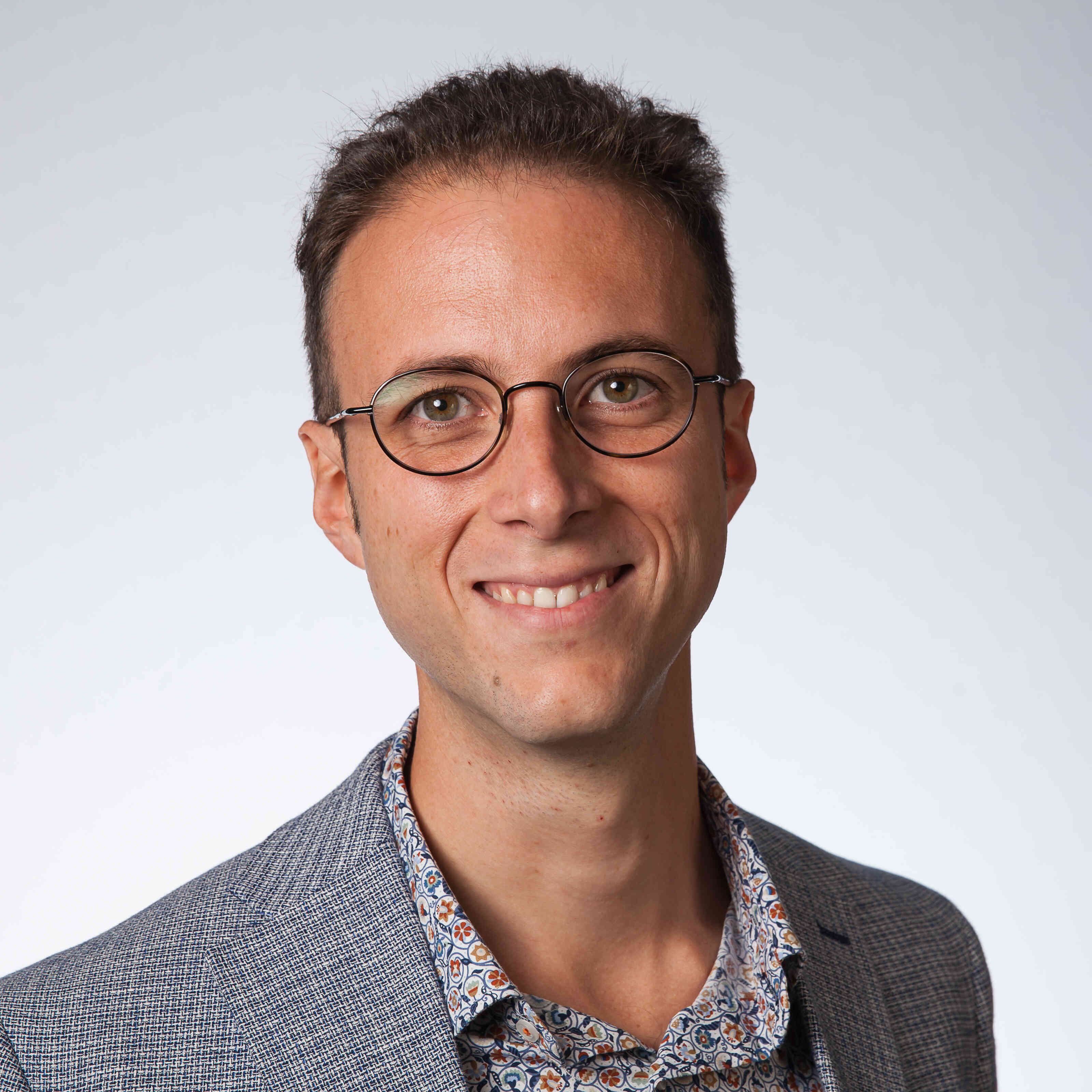}}]{Cosimo Della Santina} (Senior Member, IEEE) 
is an Assistant Professor at TU Delft and a Research Scientist at the German Aerospace Institute (DLR). He received his Ph.D. in robotics (cum laude, 2019) from the University of Pisa. He was a visiting Ph.D. student and a postdoc (2017 to 2019) at the Computer Science and Artificial Intelligence Laboratory, Massachusetts Institute of Technology (MIT). He was a senior postdoc (2020) and guest lecturer (2021) at the Department of Informatics, Technical University of Munich (TUM). Cosimo has received several awards, including the euRobotics Georges Giralt Ph.D. Award (2020) and the IEEE RAS Early Academic Career Award (2023). He is PI for TU Delft of the European Projects Natural Intelligence and EMERGE, co-director of the Delft AI Lab SELF, and involved in several Dutch projects. His research interest is in providing motor intelligence to physical systems, focusing on elastic and soft robots.
\end{IEEEbiography}
\vskip -2\baselineskip plus -1fil
\begin{IEEEbiography}[{\includegraphics[width=1in,height=1.25in,clip,keepaspectratio]{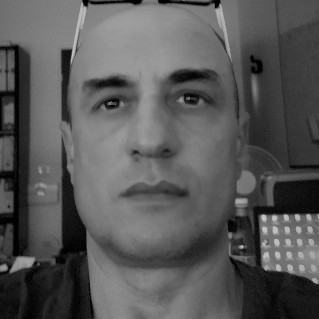}}]{Frédéric Boyer} (Member, IEEE) 
was born in France in 1967. He received the Diploma degree in mechanical engineering from the Institut Nationale Polytechnique de Grenoble, Grenoble, France, in 1991, the Master of Research degree in mechanics from the University of Grenoble, Grenoble, in 1991, and the Ph.D. degree in robotics from the University of Paris VI, Paris, France, in 1994. He is currently a Professor with the Department of Automatic Control, Institut Mines Telecom Atlantique, Nantes, France, and co-leader of the Bio-Robotics Team, Laboratoire des Sciences du Numérique de Nantes. He has coordinated several research projects including one European FP7-FET project on a reconfigurable eel-like robot able to navigate with electric sense. His current research interests include continuum and soft robotics, geometric mechanics, and biorobotics. Dr. Boyer is the recipient of the Monpetit Prize from the Academy of Science of Paris in 2007 for his work in dynamics and the French “La Recherche Prize” in 2014 for his works on artificial electric sense.
\end{IEEEbiography}
\vskip -2\baselineskip plus -1fil
\begin{IEEEbiography}[{\includegraphics[width=1in,height=1.25in,clip,keepaspectratio]{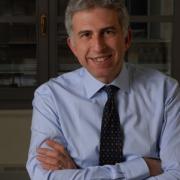}}]{Alessandro De Luca} (Life Fellow, IEEE)
is Professor of Robotics and Automation at Sapienza University of Rome. He has been the first Editor-in-Chief of the IEEE Transactions on Robotics (2004-08), RAS Vice-President for Publication Activities in 2012-13, General Chair of ICRA 2007, and Program Chair of ICRA 2016. He received three conference awards (Best paper at ICRA 1998 and BioRob 2012, Best application paper at IROS 2008), the Helmholtz Humboldt Research Award in 2005, the IEEE-RAS Distinguished Service Award in 2009, and the IEEE George Saridis Leadership Award in Robotics and Automation in 2019. He is an IEEE Life Fellow since 2023. His research interests cover modeling, motion planning, and control of robotic systems (flexible manipulators, kinematically redundant arms, underactuated robots, wheeled mobile robots), as well as physical human-robot interaction. He was the scientific coordinator of the FP7 project SAPHARI – Safe and Autonomous Physical Human-Aware Robot Interaction (2011-15).
\end{IEEEbiography}
\vskip -2\baselineskip plus -1fil
\begin{IEEEbiography}[{\includegraphics[width=1in,height=1.25in,clip,keepaspectratio]{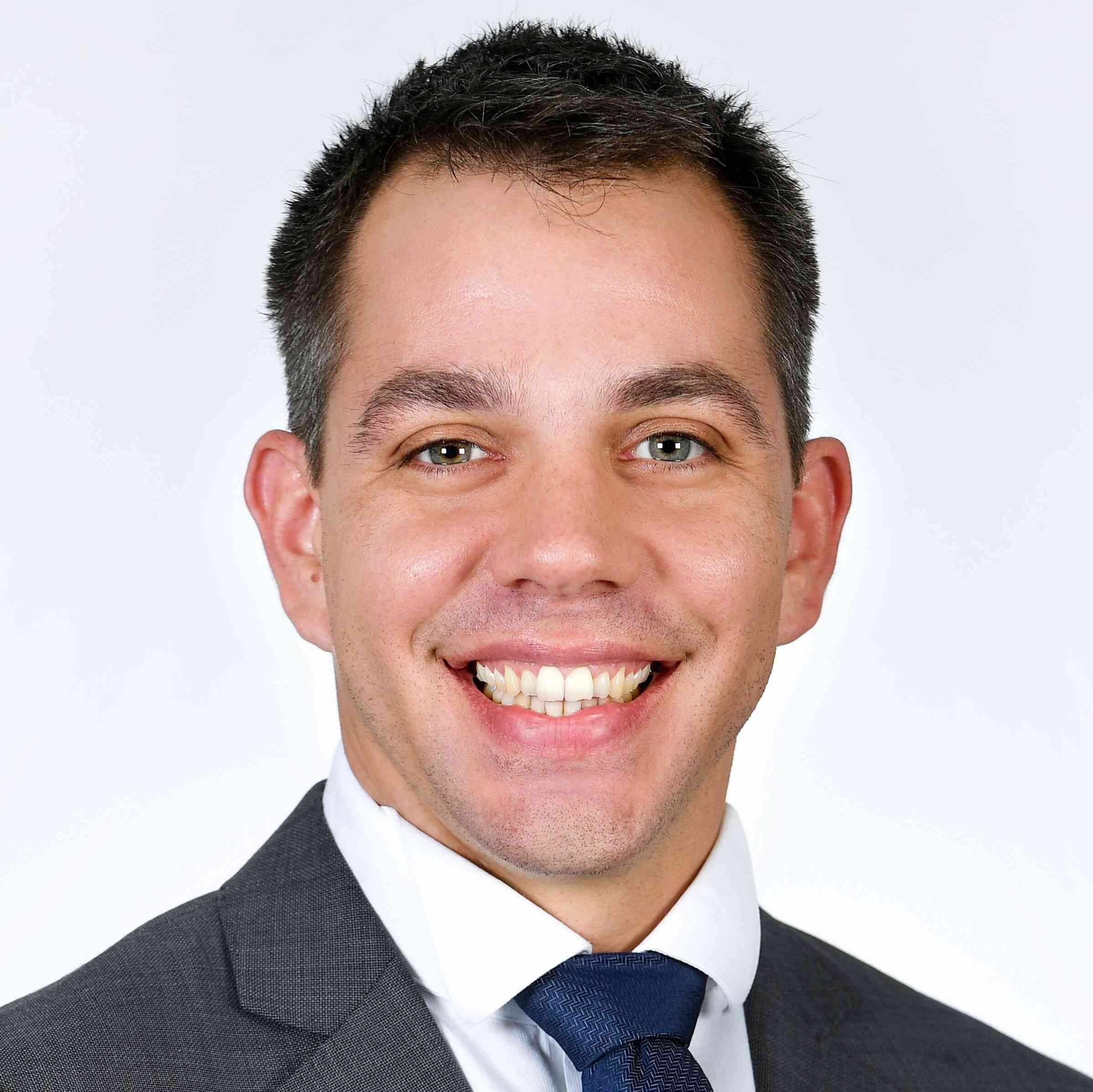}}]{Federico Renda} (Member, IEEE)
is an Associate Professor in the Department of Mechanical and Nuclear Engineering at Khalifa University, Abu Dhabi, UAE. He obtained his B.Sc. and M.Sc. degrees in biomedical engineering from the University of Pisa, Italy, in 2007 and 2009, respectively. In 2014, he earned his Ph.D. in Bio-robotics from the Biorobotics Institute, Scuola Superiore Sant’Anna, Pisa. Dr. Renda's research focuses on dynamic modeling and control of soft and continuum robots using principles of geometric mechanics.
\end{IEEEbiography}
\end{document}